\documentclass{article}

\usepackage{microtype}
\usepackage{subcaption}

\usepackage{hyperref}       
\usepackage{url}            
\usepackage{booktabs}       
\usepackage{amsfonts}       
\usepackage{nicefrac}       
\usepackage{microtype}      
\usepackage{xcolor}         
\usepackage{adjustbox}

\usepackage{amsmath, amssymb, amscd, amsthm, amsfonts}
\usepackage{graphicx}
\usepackage{hyperref}
\usepackage{listings}
\usepackage{etoolbox}   
\usepackage{verbatim}   
\usepackage{fancyvrb}
\usepackage{menukeys}
\usepackage{titlesec}
\usepackage{float}
\usepackage{algorithm, algorithmic}
\usepackage{enumitem}
\usepackage{diagbox}
\usepackage{soul}
\usepackage{wrapfig}
\usepackage{multirow}
\usepackage{xurl}

\newcommand{\precond}{\textsf{\small Cond-DP}}
\newcommand{\switch}{\textsf{\small Switch-Cond-DP}}
\newcommand{\dpsgd}{\textsf{\small DPSGD}}
\newcommand{\wtdllp}{\textsf{\small Weighted-LLP}}
\newcommand{\rronbins}{\textsf{\small RR-on-Bins$_{\eps}^{\Phi}$}}

\allowdisplaybreaks

\usepackage[accepted]{icml2026}

\usepackage{amsmath,amsfonts,bm}
\usepackage{amsthm}
\usepackage{mathtools}

\newtheorem{lemma}{Lemma}[section]
\newtheorem{theorem}[lemma]{Theorem}
\newtheorem{assumption}[lemma]{Assumption}
\newtheorem{definition}[lemma]{Definition}

\newtheorem{remark}[lemma]{Remark}

        {\hspace*{\fill}$\Box$\par}

\newtheorem{innercustomgeneric}{\customgenericname}
\providecommand{\customgenericname}{}
\newcommand{\newcustomtheorem}[2]{%
  \newenvironment{#1}[1]
  {%
  \renewcommand\customgenericname{#2}%
  \renewcommand\theinnercustomgeneric{##1}%
  \innercustomgeneric
  }
  {\endinnercustomgeneric}
}

\newcustomtheorem{customthm}{Theorem}
\newcustomtheorem{customlemma}{Lemma}
\newcustomtheorem{customproblem}{Problem}

\DeclarePairedDelimiterX{\infdivx}[2]{(}{)}{%
  #1\;\delimsize\|\;#2%
}

\def\eqref#1{equation~\ref{#1}}

\def\1{\bm{1}}

\def\eps{{\epsilon}}

\def\rvb{{\mathbf{b}}}

\def\rvg{{\mathbf{g}}}

\def\rvu{{\mathbf{i}}}

\def\rvu{{\mathbf{u}}}
\def\rvv{{\mathbf{v}}}

\def\rvx{{\mathbf{x}}}
\def\rvy{{\mathbf{y}}}

\def\mA{{\bm{A}}}
\def\mB{{\bm{B}}}
\def\mC{{\bm{C}}}
\def\mD{{\bm{D}}}

\def\mG{{\bm{G}}}

\def\mS{{\bm{S}}}

\def\mU{{\bm{U}}}
\def\mV{{\bm{V}}}
\def\mW{{\bm{W}}}
\def\mX{{\bm{X}}}

\DeclareMathAlphabet{\mathsfit}{\encodingdefault}{\sfdefault}{m}{sl}
\SetMathAlphabet{\mathsfit}{bold}{\encodingdefault}{\sfdefault}{bx}{n}

\def\gA{{\mathcal{A}}}

\def\gD{{\mathcal{D}}}
\def\gE{{\mathcal{E}}}

\def\gK{{\mathcal{K}}}
\def\gL{{\mathcal{L}}}
\def\gM{{\mathcal{M}}}
\def\gN{{\mathcal{N}}}
\def\gO{{\mathcal{O}}}

\def\gX{{\mathcal{X}}}
\def\gY{{\mathcal{Y}}}

\def\sI{{\mathbb{I}}}

\newcommand{\E}{\mathbb{E}}

\newcommand{\R}{\mathbb{R}}

\DeclareMathOperator*{\argmin}{arg\,min}

\usepackage[capitalize,noabbrev]{cleveref}

\usepackage[textsize=tiny]{todonotes}

\newcommand{\update}[1]{\textcolor{black}{#1}}

\begin{document}

\twocolumn[
  \icmltitle{Private Learning with Public Feature Conditioning}



  \icmlsetsymbol{equal}{*}
  \icmlsetsymbol{note}{*}

  \begin{icmlauthorlist}
    \icmlauthor{Shuli Jiang}{comp1,note}
    \icmlauthor{Walid Krichene}{comp2,note}
    \icmlauthor{Nicolas Mayoraz}{comp3}
  \end{icmlauthorlist}

  \icmlaffiliation{comp1}{AWS Agentic AI}
  \icmlaffiliation{comp2}{Microsoft}
  \icmlaffiliation{comp3}{Google Research}

  \icmlcorrespondingauthor{Shuli Jiang}{shulij@amazon.com}
  \icmlcorrespondingauthor{Walid Krichene}{wkrichene@microsoft.com}
  \icmlcorrespondingauthor{Nicolas Mayoraz}{nmayoraz@google.com}

  \icmlkeywords{Machine Learning, ICML}

  \vskip 0.3in
]



\printAffiliationsAndNotice{* Part of the work was done while at Google Research. \break}  

\begin{abstract}
  We study differentially private (DP) regression in settings where each data sample includes public, non-sensitive features—common in applications like recommendation or advertising systems.
    While such label DP or DP with semi-sensitive features settings have been primarily explored in the context of classification, effective approaches for regression remain underexplored.
    We introduce \textsf{\small Cond-DP}, a conditioned variant of \textsf{\small DPSGD} that leverages the structure of public feature matrices to improve optimization under privacy constraints. 
    Motivated by the observation that these public features often exhibit rapidly decaying spectra, \textsf{\small Cond-DP} incorporates a data-driven conditioning matrix to reshape the optimization landscape and accelerate convergence.
    We provide convergence guarantees for convex, strongly convex and non-convex settings, and recover standard \textsf{\small DPSGD} as a special case when the conditioning matrix is the identity.
    We show how to construct an effective conditioning matrix for \textsf{\small Cond-DP} directly from public features, enabling provably faster convergence than \textsf{\small DPSGD} in private linear regression, without incurring additional privacy cost. 
    Empirically, \textsf{\small Cond-DP} with this conditioning matrix consistently outperforms state-of-the-art baselines across a wide range of datasets and model architectures under label DP, demonstrating strong and robust performance in practice. 
\end{abstract}

\section{Introduction}
Differential privacy (DP)~\cite{dwork2014dp} has become the leading framework for training machine learning models while providing formal guarantees on the privacy of sensitive training data. However, a fundamental challenge in DP is the inherent privacy-utility trade-off: models trained with privacy guarantees often suffer a drop in utility compared to their non-private counterparts. This loss in performance can be especially critical in utility-sensitive applications such as recommendation and advertising systems, where even a small increase in model's prediction error can have significant downstream impacts on engagement.

An effective approach to improve the privacy-utility trade-off in real-world applications is to leverage public, non-sensitive features that naturally occur in many datasets.
For example, in user shopping data, item descriptions are typically public, while individual shopping histories are sensitive. 
Such scenarios are formalized under label differential privacy (label DP), where features are public and labels are sensitive~\cite{ghazi2023regression_ldp, ghazi2021deep_learning_ldp}. More recently, a notion of DP with semi-sensitive features was proposed~\cite{krichene2023priv_learning_pub_features,chua2024adkdd}, where each sample includes a mix of public and private features, and the labels remain private.
The semi-sensitive setting interpolates between label DP (in which all features are public) and the classical DP setting (in which all features are private).

While prior work has leveraged public features to improve privacy-utility trade-offs, several limitations remain. In the better-studied label DP setting, most existing approaches focus on classification tasks with discrete labels~\cite{ghazi2021deep_learning_ldp, ghazi2024labeldppro, Busa-Fekete2023label_do_priv_data_release, esfandiari2021ldp_clustering, esmaeili2021antipodes_ldp}, and are not readily applicable to regression with continuous labels.
To the best of our knowledge, only two works address regression tasks under label DP. 
The first~\cite{ghazi2023regression_ldp} provides pure DP guarantees and considers a feature-oblivious setting, where the algorithm is based solely on private labels and is independent of public features. This approach has two main drawbacks: (1) approximate DP is more widely adopted in practice due to its more favorable privacy-utility trade-offs, and this difference is often fundamental; (2) ignoring the structure of public features misses opportunities to further improve model utility. 
The second work~\cite{brahmbhatt2023ldp_aggregation} investigates a specific type of aggregation-based algorithm for private linear regression under approximate label DP. However, its goal is to study whether a certain narrow class of aggregation algorithms used in practice can be made to satisfy label DP guarantees, leaving open the question of whether better algorithms exist for this purpose.

Few works~\cite{krichene2023priv_learning_pub_features,chua2024hybrid_dp_kdd} have studied the more recent setting of DP with semi-sensitive features.
\cite{krichene2023priv_learning_pub_features} is only applicable to a specific class of model architectures (consisting of dual-encoder models). \cite{chua2024hybrid_dp_kdd} is applicable more broadly but performs poorly in high privacy regimes. Their method consists of using Randomized Response (RR)~\cite{dwork2014dp} to privatize labels, using them to train an initial model on public features only, then using this model to warm-start the full model (which is trained using DP-SGD).
The poor performance was attributed to the rapid degradation of RR in high-privacy regimes: RR-privatized labels become too noisy, making the warm-start essentially useless.

These limitations of existing approaches motivate us to ask:

\fbox{%
\parbox{0.48\textwidth}{%
\textit{
Can we better leverage the structure of public features—without jointly using private labels—in regression tasks under approximate differential privacy, to improve utility in high-privacy regimes?
}
}%
}

In this work, we focus on a \update{general class of models with linear input transformations} that are widely used in recommendation and advertising systems, as illustrated in Figure~\ref{fig:model_arch}. The model comprises two main components: a bottom linear embedding layer that learns embedding vectors for input features, and an (optional) top prediction component, consisting of potentially nonlinear layers. 
Common choices for the prediction component include multi-layer perceptrons (MLPs)~\cite{Popescu2009mlp} and Factorization Machines (FM)~\cite{rendle2010FM}.
This model includes the important case of private linear regression, where the embedding layer is of output dimension one and there is no prediction layer.
The bottom linear embedding layer plays a critical role in our setting, as its parameters are directly tied to the input features. By using knowledge about the public features, one can hope to learn the embedding layer more effectively.

One natural idea to leverage public features without jointly using the private labels is based on the observation that, in some applications, the public feature matrix (i.e., that matrix formed by stacking the public features of all training samples) lies in a low-dimensional subspace. One can then factorize this matrix and use a compact low-rank representation to reduce the number of model parameters in the bottom layer. This dimensionality reduction would, in principle, lead to less noise being added by \dpsgd\ \cite{Abadi2016dpsgd} during training (where the noise dimension is proportional to the number of model parameters),
thereby improving utility under fixed privacy guarantees. The approaches of~\cite{kairouz2021nearly,krichene2023priv_learning_pub_features} fall under this category. 
However, dimensionality reduction does not always lead to utility improvement. Notably, for the class of \update{Generalized Linear Models (GLMs)}, \update{a special case of the models considered in this work}, ~\cite{song2021private_glm} shows that when model parameters are bounded, \dpsgd\ naturally adapts to the low-rank structure of the input without the need for explicit dimension reduction. That is, the excess risk of \dpsgd\ for GLMs naturally scales with the rank of the input features rather than the ambient dimension, rendering explicit dimensionality reduction unnecessary.

\begin{figure}[h]
  \centering
  \begin{minipage}[t]{0.29\textwidth}
  \vspace{0pt}
    \centering
    \includegraphics[width=\linewidth]{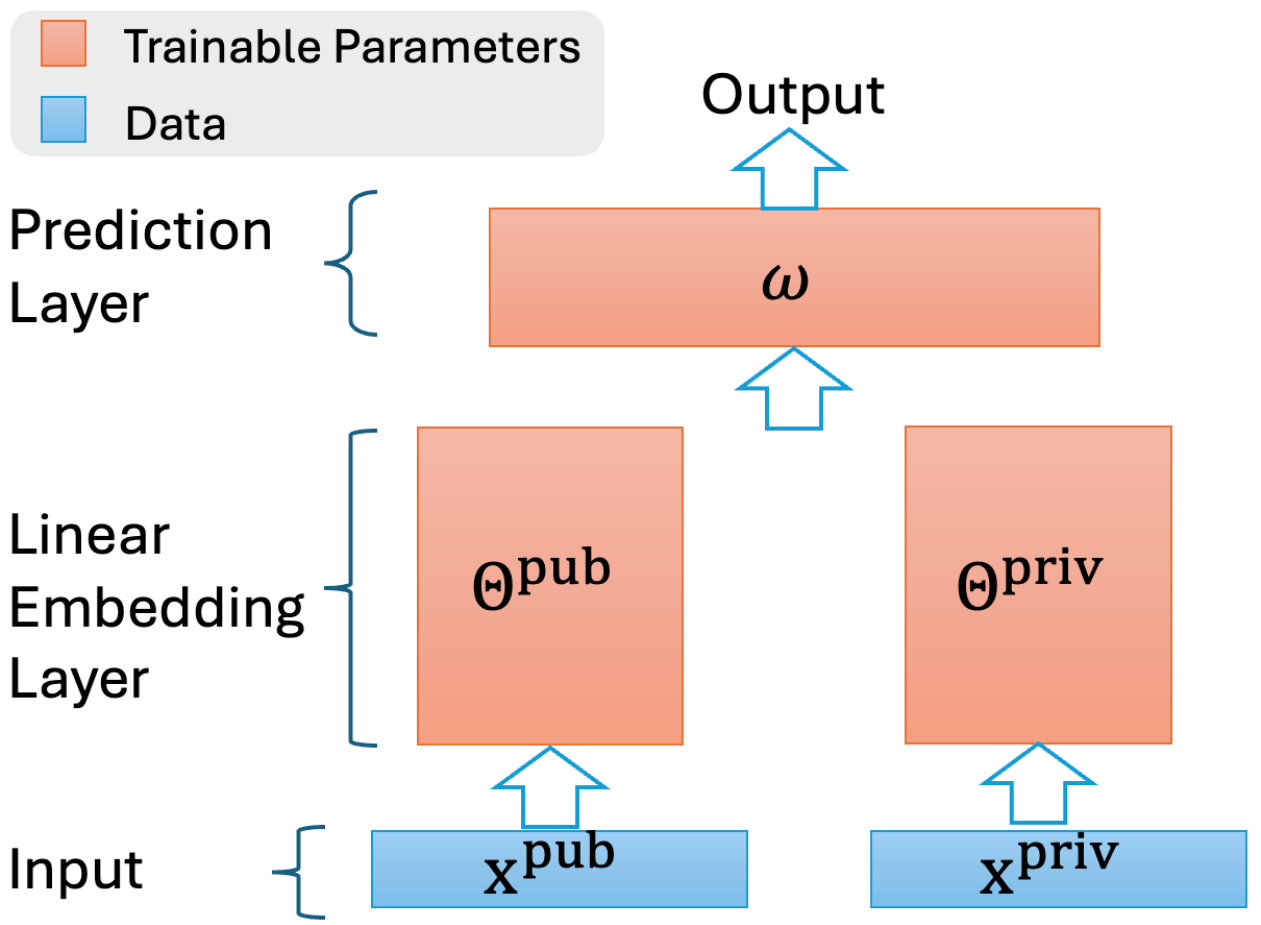}
    \caption{Model architecture.}
    \label{fig:model_arch}
  \end{minipage}\hfill
  \begin{minipage}[t]{0.18\textwidth}
  \vspace{0pt}
    \centering
    \includegraphics[width=0.84\linewidth]{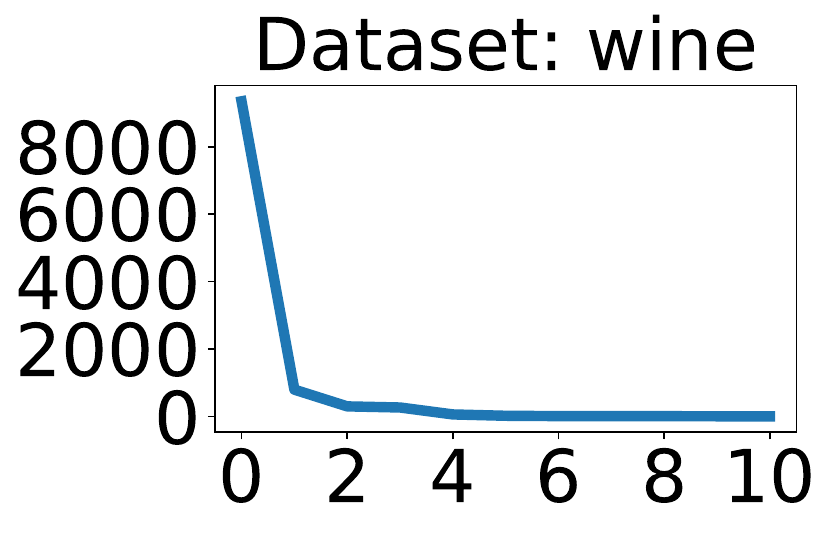}\\[2pt]
    \includegraphics[width=\linewidth]{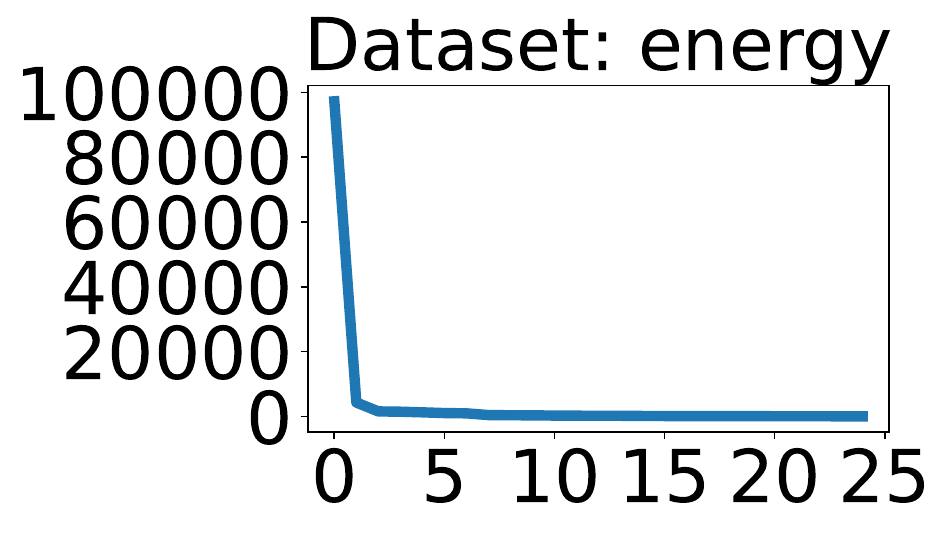}
    \caption{Example feature matrix spectra. }
    \label{fig:spectrum_datasets}
  \end{minipage}
\end{figure}

Building on the above ideas, we make a related observation: in many applications, the public feature matrix may not be strictly low-rank; but may exhibit a rapidly decaying spectrum. See
Figure~\ref{fig:spectrum_datasets} for an example taken from our numerical experiments.
Directions corresponding to the larger singular values naturally receive more weight during optimization than those associated with smaller ones.
Since \dpsgd\ perturbs gradients using isotropic noise, this results in a low signal-to-noise ratio for the low-spectrum directions, resulting in slower convergence.
We aim to address this by adapting the geometry of the problem, amplifying contributions from underrepresented directions during optimization.

Motivated by this, we propose to use knowledge about the public feature matrix to condition the parameters of the input layer, aiming to improve the signal-to-noise ratio along low-spectrum directions, and thus improving model utility—without incurring additional privacy loss.
We introduce \precond\ (Algorithm~\ref{alg:cond_dp}), which uses a conditioning matrix that transforms the parameter space associated with public feature embeddings, then trains the conditioned model using \dpsgd~\cite{Abadi2016dpsgd}.

We derive utility or convergence guarantees of $\precond$ using private \update{models with linear input transformations} under convex objectives (see Theorem~\ref{thm:convergence_glm}), strongly convex objectives (see Theorem~\ref{thm:convergence_strongly_convex}), and non-convex objectives (see Theorem~\ref{thm:convergence_non_convex}).
In the special case of private linear regression (see Lemma~\ref{lemma:improvement_C_lin_reg}), we can quantify the extent of utility improvement: we show that a carefully chosen conditioning matrix constructed from the public feature matrix provably leads to faster convergence compared to \dpsgd\, without additional privacy loss.

For more complex models, quantifying the improvement analytically would require access to the optimal solution of the underlying optimization problem, which is typically unavailable. We therefore complement our theoretical analysis with an empirical evaluation of the same conditioning strategy applied to \update{models with linear input transformations and} an MLP prediction head, as commonly used in advertising and recommendation systems.
Empirically, we observe that when using MLPs, conditioning accelerates convergence in the early stages of training but can hinder progress in later epochs. Motivated by this behavior, we propose \switch, a hybrid strategy that applies the conditioning matrix during the initial phase of training and then switches to \dpsgd\ in later epochs.

Experiments on a diverse set of datasets demonstrate that our approach consistently improves utility under fixed privacy budgets, especially in high privacy regimes. We compare against several baselines, including the state-of-the-art algorithm for regression under label DP guarantees~\cite{ghazi2023regression_ldp}, and demonstrate substantial gains in this setting.

In summary, we make the following contributions:
\begin{enumerate}[topsep=0pt]
\setlength{\itemsep}{0pt}
  \setlength{\parskip}{0pt}
  \setlength{\parsep}{0pt}
    \item We propose \precond\ (Algorithm~\ref{alg:cond_dp}), a framework that introduces a conditioning matrix to rescale the embedding space of public features.

    \item 
    We provide privacy (Theorem~\ref{thm:privacy}) and convergence guarantees for \precond\ in private \update{models with linear input transformations} under convex (Theorem~\ref{thm:convergence_glm}), strongly convex (Theorem~\ref{thm:convergence_strongly_convex}), and non-convex objectives (Theorem~\ref{thm:convergence_non_convex}). We further propose a concrete conditioning matrix constructed from public features.

    \item To effectively extend our approach to more complex, non-linear models, we introduce \switch.

    \item We empirically evaluate \precond\ and \switch\ across private linear and non-linear models in regression tasks. Experiments on both synthetic and diverse real-world datasets demonstrate consistent utility improvements over state-of-the-art baselines in the presence of public features.
\end{enumerate}
\vspace{-5pt}

\section{Related Work}

\textbf{Label Differential Privacy.}
Most prior work on label differential privacy (DP) focuses on classification and does not directly extend to regression. Early work by~\cite{Chaudhuri2011label_dp_seminal} studies sample complexity under label DP. Subsequent methods leverage Randomized Response (RR) and public information to improve utility in classification tasks, including RR with priors~\cite{ghazi2021deep_learning_ldp}, gradient projection onto low-dimensional subspaces derived from public features~\cite{ghazi2024labeldppro}, clustering using public features~\cite{esfandiari2021ldp_clustering}, modeling correlations between public features and private labels~\cite{Busa-Fekete2023label_do_priv_data_release}, and PATE-style teacher aggregation~\cite{esmaeili2021antipodes_ldp,papernot2017pate,papernot2018scalable_pate}.

Label DP for regression has received comparatively less attention. \cite{Wang2019sparse_lin_reg_local_dp} studies private linear regression under local DP by exploiting sparsity, while \cite{ghazi2023regression_ldp} proposes \rronbins, an RR-based method that bins private responses using prior label distributions, either known or privately estimated. \cite{brahmbhatt2023ldp_aggregation} explores random aggregation of samples and labels to achieve label DP guarantees in linear regression.

\textbf{Differential Privacy with Public Features.}
Recent work generalizes label DP to settings with both private and non-sensitive public features. \cite{krichene2023priv_learning_pub_features} proposes sufficient statistics perturbation, which adds noise only to parameters associated with private features, but is restricted to models with dot-product interactions between public and private embeddings. \cite{chua2024hybrid_dp_kdd} introduces a two-stage approach that first trains a model using privatized labels and public features, then fine-tunes a full model using \dpsgd, requiring a split privacy budget and showing degraded performance in high-privacy regimes.
We note that these approaches are supervised in nature. Our approach differs in that it leverages \emph{unsupervised} information from public features via conditioning, without incurring additional privacy cost, making it effective even under stringent privacy constraints.

A recent work~\cite{Saeed2025ml_privacy_protected_attr} introduces \emph{feature DP} by transforming the disclosable part of each sample and treating it as public, and adding an auxiliary public loss in a modified \dpsgd\ objective. Our approach is orthogonal, and focuses on leveraging unsupervised information. It can be readily combined with ~\cite{Saeed2025ml_privacy_protected_attr} whenever 
\update{the public component has a linear input layer.}

\vspace{-5pt}
\section{Preliminaries}

\textbf{Problem Formulation.}
We consider supervised regression tasks where data samples are drawn from an unknown distribution $\gD$ over $\gX \times \gY$, where $\gX$ is the feature space and $\gY$ are bounded real-valued labels. The feature space is decomposed as
$\gX = \gX^{\text{pub}} \times \gX^{\text{priv}}$, 
where $\gX^{\text{pub}} \subset \R^{d^{\text{pub}}}$ represents non-sensitive public features and $\gX^{\text{priv}} \subset \R^{d^{\text{priv}}}$ represents sensitive private features.
The dataset is $D = \{(\rvx_{i}^{\text{pub}}, \rvx_{i}^{\text{priv}}, y_i)\}_{i=1}^{n} \sim \gD$, where $\rvx_{i}^{\text{pub}} \in \gX^{\text{pub}}$ and $\rvx_{i}^{\text{priv}} \in \gX^{\text{priv}}$. 
The private feature dimension $d^{\text{priv}}$ may be zero, indicating the absence of private features.

\textbf{\update{Models with Linear Input Transformations.}}
Our goal is to learn a \update{general family of predictors} $g: \gX \rightarrow \R$ with the following structure.
\begin{enumerate}[leftmargin=6mm, topsep=0mm]
\setlength{\itemsep}{0pt}
  \setlength{\parskip}{0pt}
  \setlength{\parsep}{0pt}
    \item \textbf{Linear Embedding Layer.} Given an input sample $(\rvx^{\text{pub}}, \rvx^{\text{priv}}, y_i)$, the model first maps features into embeddings by computing: $\rvv^{\text{priv}} = \Theta^{\text{priv}}\rvx^{\text{priv}}$ and $\rvv^{\text{pub}} = \Theta^{\text{pub}} \rvx^{\text{pub}}$, 
    where $\Theta^{\text{pub}} \in \R^{p^{\text{pub}} \times d^{\text{pub}}}$
    and $\Theta^{\text{priv}} \in \R^{ p^{\text{priv}} \times d^{\text{priv}}}$ are learnable parameters, and $p^{\text{pub}}$, $p^{\text{priv}}$ are the embedding dimensions.
    \item \textbf{Prediction Component.}
    A function $f_{\omega}: \R^{ p^{\text{pub}}} \times \R^{ p^{\text{priv}}} \rightarrow \R$, parameterized by $\omega \in \R^{p}$, maps the embeddings to a scalar prediction. In practice, this function \update{can consist of several layers, and} can be non-linear, e.g., an MLP or a factorization machine~\cite{rendle2010FM}, commonly used in recommendation systems.
\end{enumerate}
\begin{remark}
\update{The main structural assumption we make is that the model starts with a linear layer (a common occurrence in practice) and that the private and public features can be separated at the input layer, another mild assumption -- note that any linear layer $\Theta x$ can be written as $\Theta^{\text{pub}}\rvx^{\text{pub}} + \Theta^{\text{priv}} \rvx^{\text{priv}}$ (where $x$ is the concatenation of $\rvx^{\text{pub}}, \rvx^{\text{priv}}$ and $\Theta$ is the concatenation of $\Theta^{\text{pub}}, \Theta^{\text{priv}}$).}
\end{remark}
\begin{remark}
\update{The model class considered here includes as a special case generalized linear models (GLMs)~\cite{song2021private_glm}. In classical GLMs, given an input feature vector $\rvx$, the model outputs $f_{\omega}(\langle \theta, \rvx \rangle)$ where $\theta$ is the parameter vector, whereas our model class outputs $f_{\omega}(\Theta \rvx)$, where $\Theta$ is a linear transformation parameter.}
\end{remark}

We train the model by minimizing the empirical loss,
{\small
\begin{align}
\label{eq:loss}
    \gL(\Theta^{\text{pub}}, \Theta^{\text{priv}}, \omega; &D) 
    = \frac{1}{n} \sum_{i=1}^{n} \gL_i(\Theta^{\text{pub}}, \Theta^{\text{priv}}, \omega)\\
    \nonumber
    &= \frac{1}{n}\sum_{i=1}^{n}
    \left[ l(f_{\omega}( \Theta^{\text{pub}} \rvx_i^{\text{pub}}, \Theta^{\text{priv}} \rvx_i^{\text{priv}}), y_i) \right]
\end{align}
}%
where $l(\cdot, \cdot)$ is a suitable function, for instance, $l(\widehat{y}, y) = (\widehat{y} - y)^2$.
When 
$p^{\text{priv}} = p^{\text{pub}} = 1$ and $f_{\omega}(\rvv^{\text{pub}}, \rvv^{\text{priv}}) = \rvv^{\text{pub}} + \rvv^{\text{priv}}$, the problem reduces to linear regression.

\textbf{Label DP and DP with Semi-Sensitive Features.}
Following prior work~\cite{Chaudhuri2011label_dp_seminal, chua2024hybrid_dp_kdd}, we define two datasets $D, D'$ as adjacent
if they differ in the private features (if present) and the label of a single sample
—specifically, by replacing
$(\rvx^{\text{pub}}, \rvx^{\text{priv}}, y)$ with $(\rvx^{\text{pub}}, \widetilde{\rvx}^{\text{priv}}, \widetilde{y})$.
When private features are absent, this definition naturally reduces to label differential privacy (Label DP~\cite{Chaudhuri2011label_dp_seminal}).

\begin{definition}[DP~\cite{dwork2014dp}]
\label{def:dp}
    For $\eps, \delta \geq 0$, a randomized mechanism $\gM$ satisfies $(\eps, \delta)$-DP, if for all pairs of adjacent datasets $D, D'$, and for all outcome events $S$, it holds that $\Pr[\gM(D) \in S] \leq e^{\eps}\cdot \Pr[\gM(D') \in S] + \delta$.
\end{definition}
When $\delta = 0$ (respectively $\delta > 0$), the above guarantee is referred to as pure DP (respectively approximate DP).

\textbf{Notation.}
We use $\|\rvx\|$ to denote the Euclidean norm of a vector $\rvx$ and $\|\mX\|_F$ to denote the Frobenius norm of a matrix~$\mX$. Let $\|\rvx\|_{\mA} \triangleq \sqrt{\rvx^{T}\mA \rvx}$. 
Let $\text{Tr}(\cdot)$ denote the matrix trace operator.
$\sI_d$ is the identity matrix of size $d\times d$.
Let $d = p^{\text{pub}} d^{\text{pub}} + p^{\text{priv}} d^{\text{priv}} + p$ denote the total number of model parameters.
Let $((\Theta^{\text{pub}})^{*}, (\Theta^{\text{priv}})^{*}, \omega^{*}) \in \argmin_{\Theta^{\text{pub}}, \Theta^{\text{priv}}, \omega} \gL(\Theta^{\text{pub}}, \Theta^{\text{priv}}, \omega;D)$ be optimal parameters and $\gL^*$ be the optimal loss.
Finally, $\gN_{m \times n}(0, \sigma^2)$ denotes the distribution of real $m \times n$ matrices with each entry sampled from the Gaussian distribution~$\gN(0, \sigma^2)$.

\section{$\precond$: Conditioning Model Parameters with Public Features}

We propose \precond\ (Algorithm~\ref{alg:cond_dp}), a variant of \dpsgd\ \cite{Abadi2016dpsgd} that incorporates a conditioning matrix
$\mC \in \R^{d^{\text{pub}} \times d^{\text{pub}}}$  to improve the optimization landscape and potentially accelerate convergence when learning public feature embeddings.

The matrix $\mC$ can be thought of as a hyperparameter, fixed throughout training, and does not incur any additional privacy cost, as it will be computed based only on public information. 
We discuss the choice of $\mC$ to improve convergence later in this section.
Specifically, in \precond, given an input sample $(\rvx^{\text{pub}}, \rvx^{\text{priv}}, y)$, 
we replace the standard public embedding computation
$\rvv^{\text{pub}} = \Theta^{\text{pub}} \rvx^{\text{pub}}$
with the conditioned version
$\rvv^{\text{pub}} = \Theta^{\text{pub}} \mC \rvx^{\text{pub}}$.
The model is then trained using \dpsgd\ by minimizing the loss function $\gL(\Theta^{\text{pub}} \mC, \Theta^{\text{priv}}, \omega; D)$. 
Figure~\ref{fig:model_C} illustrates the model with the conditioning matrix $\mC$ in \precond.

\begin{figure}[t]
    \centering
    \includegraphics[width=0.6\linewidth]{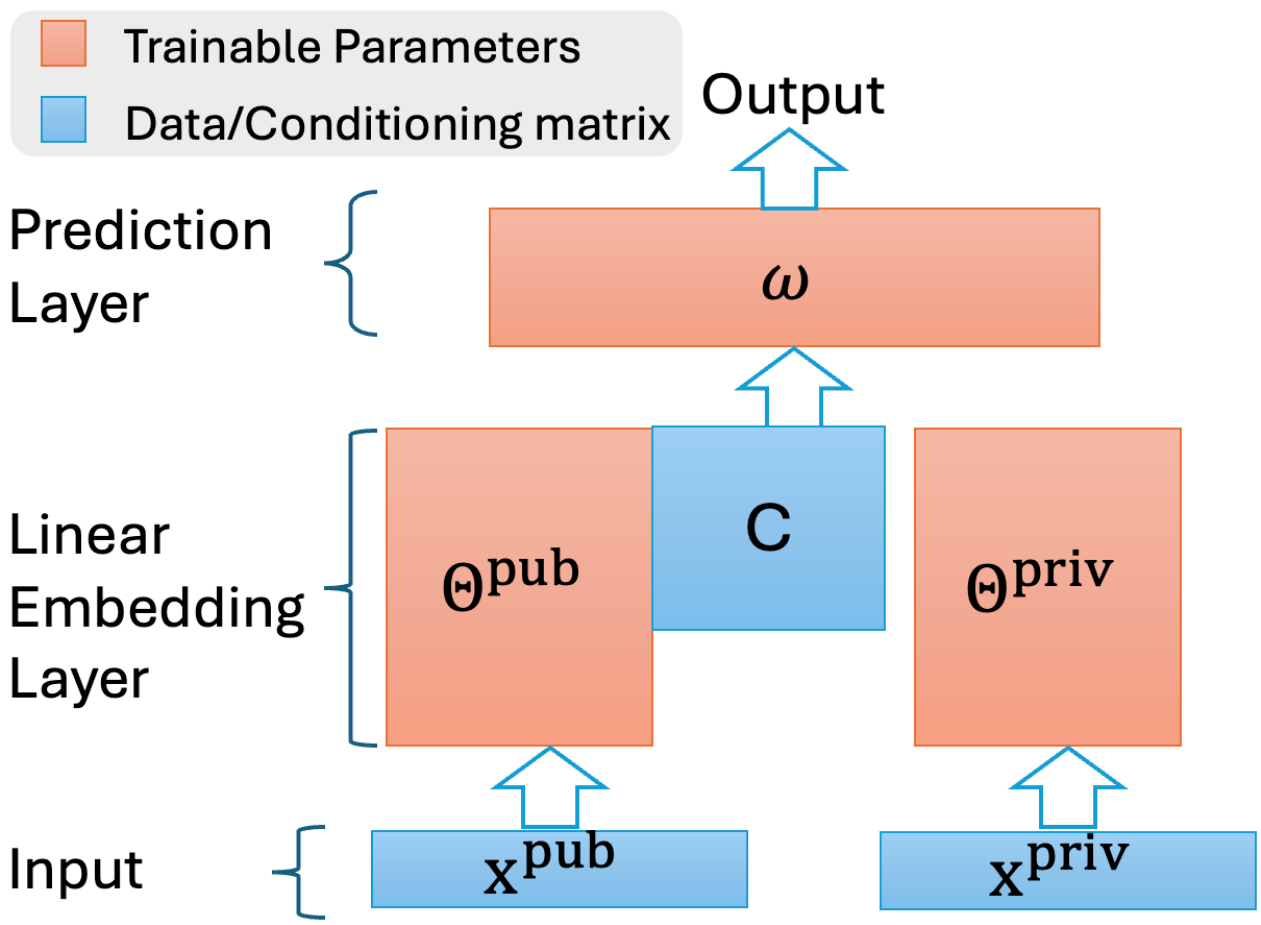}
    \caption{Model with conditioning matrix $\mathbf{C}$ trained by \precond.}
    \vspace{-10pt}
    \label{fig:model_C}
\end{figure}

  \begin{algorithm}[h]
    \caption{$\precond$}
    \label{alg:cond_dp}
    \begin{algorithmic}
        \STATE Input: dataset $D=\{(\rvx_i^{\text{pub}}, \rvx_i^{\text{priv}}, y_i)\}_{i=1}^{n}$, conditioning matrix $\mC$,
        initialization $\Theta_0^{\text{pub}}, \Theta_0^{\text{priv}}, \omega_0$,
        number of iterations $T$
        learning rate $\{\eta_t\}_{t=1}^{T}$,
        noise variance~$\sigma^2$
        \FOR{$t=0,1,\dots, T-1$}
            \STATE $l_{t, i}\leftarrow \gL_i(\Theta_t^{\text{pub}}\mC, \Theta_t^{\text{priv}}, y_i), \forall i\in [n]$
            \STATE Compute gradient: $\mG_{t}^{\text{pub}} \leftarrow \frac{1}{n}\sum_{i=1}^{n} \frac{\partial l_{t, i}}{\partial \Theta_t^{\text{pub}}}$,
            $\mG_t^{\text{priv}} \leftarrow \frac{\partial l_{t, i}}{\partial \Theta_t^{\text{priv}}}$,
            $\rvg_t^{\text{top}} \leftarrow \frac{1}{n}\sum_{i=1}^{n} \frac{\partial l_{t, i}}{\partial \omega_t}$
            \STATE Sample noise: $\mB_t^{\text{pub}} \sim \gN_{p^{\text{pub}} \times d^{\text{pub}}}(0, \sigma^2)$, $\mB_t^{\text{priv}} \sim \gN_{p^{\text{priv}} \times d^{\text{priv}}}(0, \sigma^2)$, $\rvb_t^{\text{top}} \sim \gN(0, \sigma^2 \sI_{p})$
            \STATE {\small $\Theta_{t+1}^{\text{pub}} \leftarrow \Theta_{t}^{\text{pub}} - \eta_t (\mG_t^{\text{pub}} + \mB_t^{\text{pub}}),
            \Theta_{t+1}^{\text{priv}} \leftarrow \Theta_t^{\text{priv}} - \eta_t (\mG_t^{\text{priv}} + \mB_t^{\text{priv}}),
            \omega_{t+1} \leftarrow \omega_t - \eta_t (\rvg_t^{\text{top}} + \rvb_t^{\text{top}})$
            }
        \ENDFOR
        \STATE Return $\Theta^{\text{pub-out}} = \frac{1}{T}\sum_{t=1}^{T} \Theta_t^{\text{pub}}$,
        $\Theta^{\text{priv-out}} = \frac{1}{T}\sum_{t=1}^{T} \Theta_t^{\text{priv}}$,
        $\omega^{\text{out}} = \frac{1}{T}\sum_{t=1}^{T} \omega_t$
    \end{algorithmic}
    \vspace{-1pt}
  \end{algorithm}

\subsection{Privacy and Utility Guarantees}

We first present the privacy guarantee in Theorem~\ref{thm:privacy} (full proof in Appendix~\ref{sec:appendix_privacy}), which applies to all types of losses. 
Then, we present the utility analysis for 
$\precond$ in terms of convergence in private \update{models with linear input transformations} as follows under different types of losses: 1) convex loss in Theorem~\ref{thm:convergence_glm} (full proof in Appendix~\ref{sec:appendix_convergence_glm}), 2) strongly convex loss in Theorem~\ref{thm:convergence_strongly_convex} (full proof in Appendix~\ref{sec:appendix_convergence_strongly_convex}), 
and 3) non-convex loss in Theorem~\ref{thm:convergence_non_convex} (full proof in Appendix~\ref{sec:appendix_convergence_non_convex}).

\begin{assumption}[Fine-grained Lipschitz condition
\footnote{This assumption is required solely for the privacy analysis (not for convergence) and is common in DP.}
]
\label{ass:lipschitzness}
    There exist constants $\widehat{G}, \overline{G}, G, \update{R} > 0$ such that
    $\|\frac{\partial \gL}{\partial \rvv^{\text{priv}}}\|^2 \leq \widehat{G}^2$,
    $\|\frac{\partial \gL}{\partial \omega}\|^2 \leq \overline{G}^2$
    and $\|\frac{\partial \gL}{\partial \rvv^{\text{pub}}}\|^2 \leq G^2$, for any $\rvv^{\text{priv}}$, $\omega$ and $\rvv^{\text{pub}}$\update{, and $\|\rvx^{\text{priv}}\|^2 \leq R^2$, for any private input features $\rvx^{\text{priv}}$}.
\end{assumption}

\begin{remark}[On the effect of conditioning on the fine-grained Lipschitz condition]
\update{In Assumption~\ref{ass:lipschitzness}, although the derivative $\partial \mathcal L / \partial \omega$ may appear to depend on $\mC$, it is natural to define a Lipschitz bound that is uniform in $\mC$.
Indeed, define $\tilde{\mathcal L}(\Theta^{\text{pub}}, \Theta^{\text{priv}}, \omega) = \mathcal L(\mC\Theta^{\text{pub}}, \Theta^{\text{priv}}, \omega)$, then $\frac{\partial \tilde {\mathcal L}}{\partial \omega}(\Theta^{\text{pub}}, \Theta^{\text{priv}}, \omega) = \frac{\partial \mathcal L}{\partial \omega}(\mC\Theta^{\text{pub}}, \Theta^{\text{priv}}, \omega)$. The key observation is that if $\|\frac{\partial \mathcal L}{\partial \omega}(\Theta^{\text{pub}}, \Theta^{\text{priv}}, \omega)\|$ is uniformly bounded over all $\Theta$, then $\|\frac{\partial \mathcal L}{\partial \omega}(\mC\Theta^{\text{pub}}, \Theta^{\text{priv}}, \omega)\|$ is bounded by the same constant over all $\Theta$ since $\Theta \mapsto \mC\Theta$ maps $\mathbb R^d$ to $\mathbb R^d$. 
We focus on the unconstrained $\Theta$ setting for analytical tractability. Nevertheless, the same argument extends to constrained settings (e.g., when $\Theta$ is restricted to a convex set $\gK$), provided the constraint set is transformed accordingly. Specifically, when optimizing the conditioned objective $\tilde{\mathcal L}$, the parameter $\Theta$ should be constrained to
$\gK_C = \{\Theta | \mC\Theta \in \gK\}$, which naturally preserves the geometry of the original constraint set. 
Under this modification, any bound on
$\|\frac{\partial \mathcal{L}}{\partial \omega}\|$ for $\Theta \in \gK$ directly yields the same bound on $\|\frac{\partial \tilde{\mathcal{L}}}{\partial \omega}\|$ for $\Theta \in \gK_C$.}
\end{remark}

\begin{theorem}[Privacy guarantee]
\label{thm:privacy}
    If the noise variance in Algorithm ~\ref{alg:cond_dp} is set to be $\sigma^2 = \widetilde{O}(\frac{M^2 T}{\eps^2 n^2})$, where $M^2 \triangleq G^2 \cdot \max_{i\in [n]} \|\mC\rvx_i^{\text{pub}}\|^2 
    + \widehat{G}^{2} \cdot \update{R^2}
    + \overline{G}^2$, then Algorithm~\ref{alg:cond_dp} is $(\eps, \delta)$-differentially private.
\end{theorem}

\begin{remark}
Since $\mC$ affects the bound on gradient norms, note that the noise variance $\sigma^2$, which scales with the gradient norm, also depends on $\mC$ via $M$.
\end{remark}

\begin{assumption}
\label{ass:loss_glm}
    The loss $\gL_i$ is convex $\forall i\in [n]$.
\end{assumption}

\begin{theorem}[Convergence, Convex]
\label{thm:convergence_glm}
    Let $\mC$ be a conditioning matrix of full rank. Under Assumptions~\ref{ass:lipschitzness} and~\ref{ass:loss_glm}, if Algorithm~\ref{alg:cond_dp} is run with the learning rate $\eta_t = \eta \propto \frac{1}{\sqrt{TM^2(1 + \frac{dT}{n^2\eps^2})}}$, $\forall t$, and noise variance $\sigma^2 = \widetilde{O}(\frac{M^2 T}{\eps^2 n^2})$ for $M$ defined in Theorem~\ref{thm:privacy} 
    , then it guarantees
    {\setlength{\abovedisplayskip}{4pt}
 \setlength{\belowdisplayskip}{4pt}
    {\small
    \begin{align}
    \label{eq:convergence_convex_losses}
        &\E\big[\gL(\Theta^{\text{pub-out}} \mC, \Theta^{\text{priv-out}}, \omega^{\text{out}}; D)\big]
        - \gL^*\\
        \nonumber
        &\leq M \sqrt{\frac{1}{T} + \frac{d}{n^2\eps^2} } 
        \cdot \Big(  \Big\| \Theta_0^{\text{pub}} \mC - {(\Theta^{\text{pub}}})^{*} \Big\|^2_{F, \mC}\\
        \nonumber
        &\quad + \|\Theta^{\text{priv}}_0 - (\Theta^{\text{priv}})^*\|_F^2
        + \|\omega_0 - \omega^*\|^2 \Big)^{1/2},
    \end{align}
    }%
    }%
    where $\|\Theta\|_{F,\mC}^2 \coloneq Tr(\Theta^\top (\mC^\top \mC)^{-1} \Theta)$, and the expectation is taken w.r.t. the randomness of the algorithm. 
\end{theorem}

\begin{assumption}
\label{ass:loss_strongly_convex}
    The loss function $\gL_i$ is $\mu$-strongly convex and $\beta$-smooth, $\forall i\in [n]$.
\end{assumption}

\begin{theorem}[Convergence, Smooth and Strongly Convex]
\label{thm:convergence_strongly_convex}
    Let $\mC$ be a conditioning matrix of full rank. 
    Under Assumptions~\ref{ass:lipschitzness} and~\ref{ass:loss_strongly_convex},
    if Algorithm~\ref{alg:cond_dp} is run with learning rate $\eta_t \propto \frac{1}{\mu t}$, $\forall t$,
    and noise variance $\sigma^2 = \widetilde{O}(\frac{M^2 T}{\eps^2 n^2 })$ for $M$ defined in Theorem~\ref{thm:privacy}
    , then it guarantees
    {\setlength{\abovedisplayskip}{4pt}
 \setlength{\belowdisplayskip}{4pt}
    {\small
    \begin{align}
        &\E\left[\gL(\Theta^{\text{pub-out}}\mC, \Theta^{\text{priv-out}}, \omega^{\text{out}} ; D)\right]
        -\gL^{*}
        \leq \frac{16 \beta^2 }{\mu^2}M^2
       \\
        \nonumber
        &\quad \cdot  \left(\frac{1}{T} + \frac{d}{n^2\eps^2}\right) \left(\frac{\sigma_{max}(\mC)}{\sigma_{min}(\mC)}\right)^2
        \max\{\sigma_{max}^2(\mC), \sigma^{-2}_{min}(\mC)\}
    \end{align}
    }%
    }
    where $\sigma_{min}, \sigma_{max}$ denote the minimum and maximum singular value of $\mC$, and the expectation is taken w.r.t. the randomness of the algorithm.
\end{theorem}

\begin{assumption}[Smoothness w.r.t. a Semi-norm]
\label{ass:smoothness_semi_norm}
For all $i \in [n]$, the loss function $\gL_i$ is $\beta_0$-smooth with respect to the semi-norm induced by $(\mW \mW^{\top})^{-1}$. Specifically, for any $\theta_1, \theta_2 \in \R^{d}$,
\begin{align*}
    \| \nabla \gL_i(\theta_1) - \nabla \gL_i(\theta_2) \|_{\mW \mW^{\top}}
    \leq \beta_0 \, \| \theta_1 - \theta_2 \|_{(\mW \mW^{\top})^{-1}}
\end{align*}
Here, $\theta_1$ and $\theta_2$ denote the vectorized model parameters, and $\mW$ is a positive semidefinite matrix determined by the conditioning matrix $\mC$ (see Appendix~\ref{sec:appendix_convergence_glm} for definition).
\end{assumption}

\begin{theorem}[Convergence, Non-convex]
\label{thm:convergence_non_convex}
    Let $\mC$ be a conditioning matrix of full rank.
    Under Assumptions~\ref{ass:lipschitzness} and~\ref{ass:smoothness_semi_norm}, if Algorithm~\ref{alg:cond_dp} is run with $\eta_t = \eta = \frac{1}{\sqrt{T} \beta_0}, \forall t$, and noise variance $\sigma^2 = \widetilde{O}(\frac{M^2 T}{\eps^2 n^2 })$, then it guarantees
    {\small
    \begin{align}
    \label{eq:convergence_non_convex_losses}
        &\mathbb{E}[\|\nabla \gL(\Theta^{\text{pub-out}}\mC, \Theta^{\text{priv-out}}, \omega^{\text{out}} ; D)\|^2]\\
        \nonumber
        &\leq \frac{\beta_0}{\sqrt{T}} \Big(\gL(\Theta_0^{\text{pub}}\mC, \Theta_0^{\text{priv}}, \omega_0; D) - \gL^{*} \Big)\\
        \nonumber
        &\quad + M \Big(\frac{\beta_0}{2\sqrt{T}} + \frac{d}{2n^2\epsilon^2 \sqrt{T}} \Big)
    \end{align}
    }%
    where the expectation is taken w.r.t. randomness of the algorithm.
\end{theorem}

\begin{remark}
    Setting the conditioning matrix to the identity, i.e., $\mC = \sI_{p^{\text{pub}}}$,  in Algorithm~\ref{alg:cond_dp} recovers the standard baseline that applies \dpsgd\ directly to both public and private features, along with private labels.
     In this case, all bounds in Theorem~\ref{thm:convergence_glm}, ~\ref{thm:convergence_strongly_convex} and~\ref{thm:convergence_non_convex} reduce to the standard bounds for \dpsgd, \update{e.g., see~\cite{song2021private_glm} and~\cite{Bassily2014private_erm}}.
\end{remark}

\begin{remark}[Full DP and Label DP]
    At one extreme, when there are no public features, the setting reduces to the standard full DP scenario, and there is no conditioning.
    At the other extreme, when there are no private features, the setting reduces to label DP. In this case,
    $\mC$ can be effectively leveraged to minimize the bound on excess risk.
\end{remark}

How does the conditioning matrix $\mC$ affect the bound on excess risk in (\ref{eq:convergence_convex_losses})? There are two factors that depend on $\mC$. The first factor $$T_1\coloneq M = \sqrt{G^2 \cdot \max_{i\in [n]} \|\mC\rvx_i^{\text{pub}}\|^2 
    + \widehat{G}^{2} \cdot 
    \update{R^2}
    + \overline{G}^2}$$
captures a bound on the gradient norm, while the second factor $$T_2 \coloneq \|(\Theta_0^{\text{pub}} \mC - (\Theta^{\text{pub}})^{*})\|_{F,\mC}^2$$ captures a bound on the initial distance to optimum, measured in a scaled Frobenius norm. 
It is instructive to consider the simpler case of label DP, in which the constants $\widehat{G}, \overline{G}$ in $T_1$ are both zero, and the terms involving $\Theta_0^{\text{priv}}, \omega_0$ in (\ref{eq:convergence_convex_losses}) are also 0. In this case, we observe that the \emph{scale} of $\mC$ has no effect on the final bound, as scaling $\mC$ by a constant $k$ scales $T_1$ by $k$ and $T_2$ by $1/k$. The \emph{spectrum} of $\mC$, however, does affect $T_1, T_2$.
In particular, $T_1$ becomes smaller whenever the public feature vectors $\rvx_i^{\text{pub}}$ are aligned with the lower singular directions of $\mC$, while $T_2$ becomes smaller whenever the matrix $(\Theta_0^{\text{pub}} \mC - {(\Theta^{\text{pub}}})^{*})^\top (\Theta_0^{\text{pub}} \mC - {(\Theta^{\text{pub}}})^{*})$ is nearly orthogonal to $(\mC^\top \mC)^{-1}$.

In principle, one could treat this as an optimization problem and attempt to choose a $\mC$ that minimizes the combined bound, though this cannot be done in general without additional assumptions, due to the dependence on the unknown ${(\Theta^{\text{pub}}})^*$.
However, for certain classes of problems, it is possible to choose a $\mC$ that is guaranteed to improve the bound (even though it may not be necessarily optimal). One such class is private linear regression, in which  $(\Theta^{\text{pub}})^{*}$ admits a closed-form solution. Building on this strategy, we show, in the next section how to construct $\mC$ from the public features in a way that yields a provably better bound than \dpsgd.

\subsection{Private Linear Regression}

In this setting, the embedding dimension is $p^{\text{pub}} = p^{\text{priv}} = 1$. 
We write $\theta^{\text{pub}} = (\Theta^{\text{pub}})^T \in \R^{d^{\text{pub}}}$ and $\theta^{\text{priv}} = (\Theta^{\text{priv}})^T \in \R^{d^{\text{priv}}}$, as the parameters associated with public and private features, respectively.
The top layer is simply the sum $f_{\omega}(\Theta^{\text{pub}}\rvx^{\text{pub}}, \Theta^{\text{priv}}\rvx^{\text{priv}}) = \langle \theta^{\text{pub}}, \rvx^{\text{pub}}\rangle + \langle \theta^{\text{priv}}, \rvx^{\text{priv}} \rangle$ without additional parameters, i.e., $\omega = \emptyset$.
Let the public feature matrix be $\mX^{\text{pub}} = [(\rvx_1^{\text{pub}})^T, \dots, (\rvx_{n}^{\text{pub}})^{T}]^T$, the private feature matrix be $\mX^{\text{priv}} = [(\rvx_1^{\text{priv}})^T, \dots, (\rvx_n^{\text{priv}})^T]^T$ and the label vector be $\rvy = [y_1,\dots, y_n]^T$. The objective is then
{\setlength{\abovedisplayskip}{4pt}
 \setlength{\belowdisplayskip}{4pt}
\begin{align}
\label{eq:lin_reg_problem}
    \min_{\theta^{\text{pub}}, \theta^{\text{priv}}} \left\| \left( \mX^{\text{pub}}\theta^{\text{pub}} + \mX^{\text{priv}}\theta^{\text{priv}} \right) - \rvy \right\|^2.
\end{align}
}

Let $\mU\Sigma \mV^T$ be the singular value decomposition (SVD) of $\mX^{\text{pub}}$. We propose to set the conditioning matrix as
{\setlength{\abovedisplayskip}{4pt}
 \setlength{\belowdisplayskip}{4pt}
\begin{align}
\label{eq:cond_matrix}
    \widehat{\mC} \coloneq \mV \Sigma^{-1} \mV^T \in \R^{d^\text{pub} \times d^\text{pub}}.
\end{align}
}%
The following lemma presents the utility bound of $\precond$ with the proposed conditioning matrix $\widehat{\mC}$, as well as the \dpsgd\ baseline  (which corresponds to $\mC = \sI_{d^{\text{pub}}}$).
See Appendix~\ref{sec:appendix_priv_lin_reg} for a full proof.

\begin{lemma}
\label{lemma:improvement_C_lin_reg}
   Consider the linear regression problem defined in (\ref{eq:lin_reg_problem}), and let the singular value decomposition (SVD) of the public feature matrix be $\mX^{\text{pub}} = \mU \Sigma \mV^{T}$.
    Suppose $\precond$ (Algorithm~\ref{alg:cond_dp}) is run with $\eta_t, \sigma^2$ as in Theorem~\ref{thm:convergence_glm}, then its excess risk is bounded as follows. If $\mC = \widehat{\mC} = \mV \Sigma^{-1}\mV^{T}$, then
    {\setlength{\abovedisplayskip}{4pt}
 \setlength{\belowdisplayskip}{4pt}
    {\small
    \begin{align}
    \label{eq:convergence_precond_lin_reg}
        &\E\big[\gL(\mC^{T} \theta^{\text{pub-out}}, \theta^{\text{priv-out}}; D)\big]
        - \gL^*
        \leq \sqrt{\frac{1}{T} + \frac{d}{n^2 \eps^2}}\\
        \nonumber
        &\quad \cdot \sqrt{ 2 \| \widehat{\rvy} \|^2 + F^2}
        \cdot
        \sqrt{
            G^2 + E^2
        },
    \end{align}
    }%
    }
    and if $\mC = \sI_{d^{\text{pub}}}$ (corresponding to \dpsgd), then
    {\setlength{\abovedisplayskip}{4pt}
 \setlength{\belowdisplayskip}{4pt}
    {\small \begin{align}
    \label{eq:convergence_dpsgd_lin_reg}
        &\E\left[\gL(\theta^{\text{pub-out}}, \theta^{\text{priv-out}}; D)\right]
        - \gL^*
        \leq 
        \sqrt{\frac{1}{T} + \frac{d}{n^2 \eps^2}}
        \\
        \nonumber
        &\quad \cdot \sqrt{2 \|\Sigma^{-1} \widehat{\rvy}\|^2 + F^2
        }
        \cdot \sqrt{G^2 \sigma_{\text{max}}^2(\mX^{\text{pub}}) + E^2},
    \end{align}  }%
    }%
    where {\small $\sigma_{\text{max}}(\mX^{\text{pub}})$} is the largest singular value of $\mX^{\text{pub}}$,
    \begin{align*}
        F^2 &= 2\|\theta_0^{\text{pub}}\|^2 
        + \|\theta_0^{\text{priv}} - (\theta^{\text{priv}})^{*}\|^2,\\
        E^2 &= \widehat{G}^2\cdot \update{R^2},\\
        \widehat{\rvy} &= \Big( \sI_n
        + \mU^T \mX^{\text{priv}} \big[ (\mX^{\text{priv}})^T (\sI_n - \mU \mU^T) \mX^{\text{priv}} \big]^{-1} \\
        &\quad (\mX^{\text{priv}})^T \mU \Big)
        \mU^T \rvy
    \end{align*}
    \vspace{-20pt}
\end{lemma}

\textbf{Comparison.} 
We compare the convergence bounds of \precond\ and the \dpsgd\ baseline in (\ref{eq:convergence_precond_lin_reg}) and (\ref{eq:convergence_dpsgd_lin_reg}), under the label DP setting with no private features.
In this case, $F^2 = 2\| \theta_0^{\text{pub}}\|^2$, $E^2 = 0$, and $\widehat{\rvy} = \mU^T \rvy$.
Let the diagonal entries of $\Sigma$ be $\sigma_{max} = \sigma_1 \geq \sigma_2 \geq \dots \geq \sigma_{d^{\text{pub}}} = \sigma_{min}$. 

We first consider the zero-initialization case, i.e. $F^2=\|\theta_{0}^{\text{pub}}\|^2 =~0$. Besides the common factors $G$ and $\sqrt{\frac{1}{T} + \frac{d}{n^2\eps^2}}$ in both convergence bounds,
\dpsgd's bound scales with 
$\sqrt{2 \sum_{i=1}^{d^{\text{pub}}} (\widehat{y}_i \frac{\sigma_{max}}{\sigma_i})^2 }$; while the bound for \precond\ scales as $\sqrt{2\sum_{i=1}^{d^{\text{pub}}} \widehat{y}_i^2}$.
Since $\frac{\sigma_{max}^2}{\sigma_i^2} \geq 1, \forall i \in [d^{\text{pub}}]$, \precond\ achieves a provably tighter bound. The improvement, $$\sqrt{\sum_{i=1}^{d^{\text{pub}}} (\widehat{y}_i \frac{\sigma_{max}}{\sigma_i})^2  / \sum_{i=1}^{d^{\text{pub}}} \widehat{y}_i^2 },$$ 
is maximized when $\rvy$ aligns with the singular direction corresponding to $\sigma_{min}$ and minimized (equal to 1) when $\rvy$ aligns with the top singular direction $\sigma_{max}$.
More generally, larger gains arise when $\rvy$ aligns with lower singular directions of $\mX^{\text{pub}}$.

We next consider non-zero initialization. Let $\theta_0^{\text{DP-SGD}}$ and $\theta_0^{\text{Cond-DP}}$ denote the initialization of \dpsgd\ and \precond, respectively.
Up to shared factors, the \dpsgd\ bound scales as
$$\sqrt{2 \sum_{i=1}^{d^{\text{pub}}} (\widehat{y}_i \frac{\sigma_{max}}{\sigma_i})^2
+  \|\theta_0^{\text{DP-SGD}}\|^2\sigma_{max}^2},$$
while the \precond\ bound scales as $$\sqrt{2 \|\widehat{\rvy}\|^2 + \|\theta_0^{\text{Cond-DP}}\|^2}.$$ 
Since the initialization is an algorithmic choice, selecting
$\|\theta_0^{\text{Cond-DP}}\| \leq \|\theta_0^{\text{DP-SGD}}\|\sigma_{\max}$
is sufficient to ensure that \precond\ again achieves a strictly tighter convergence bound, with maximal improvement when $\rvy$ aligns with the smallest singular direction.

\subsection{More Complex Models and \switch}

The analysis of the private linear regression case relied on a relatively tractable dependence of the solution on the public feature matrix $\mX^{\text{pub}}$ (see Lemma~\ref{lemma:lin_reg_optimum} in Appendix~\ref{sec:appendix_priv_lin_reg}), enabling a provable improvement. For more complex models, such characterizations are typically unavailable, making a similar theoretical analysis intractable. 
We hence turn to empirical evaluation, using the conditioning matrix $\widehat{\mC}$ from (\ref{eq:cond_matrix}) under more complex settings.

A key empirical observation is that while conditioning can substantially accelerate convergence in the early stages of training, it can hinder further loss reduction in later epochs. Motivated by this behavior, we propose \switch, a hybrid training strategy that applies \precond\ during the initial phase of optimization and then switches to \dpsgd\ to facilitate continued progress.
The switching epoch is treated as a tunable hyperparameter.

\section{Experiments}
\label{sec:exp}

We empirically evaluate \precond\ under label DP across three \update{regression} settings: (i) private linear \update{models} on synthetic and real-world datasets (Sections~\ref{subsec:private_lin_reg_synthetic_ds} and~\ref{subsec:private_lin_reg_real_world_ds}); (ii) private \update{models with linear input transformation and} MLP-based prediction layers in more general non-linear settings (Section~\ref{subsec:general_glms}); and (iii) \update{private linear models, as well as private models with linear input transformation and MLP-based prediction layers} on the \texttt{Criteo Sponsored Search Conversion} benchmark (Section~\ref{subsec:criteo_search}).

\textbf{Baselines.} We compare \precond\ and \switch\ with the conditioning matrix $\widehat{\mC}$ against the following baselines: (i) \textbf{\textsf{\dpsgd}}~\cite{Abadi2016dpsgd}, which directly trains on both public and private features using a DP optimizer; and (ii) the state-of-the-art label-DP method for regression, \textbf{\rronbins}~\cite{ghazi2023regression_ldp}. See Appendix~\ref{sec:appendix_comparison_rronbins} for a detailed algorithmic comparison with \rronbins.
We also evaluate \textbf{\wtdllp}~\cite{brahmbhatt2023ldp_aggregation}, but omit its results on real-world datasets from the main text due to limited applicability and consistently inferior performance; full results are reported in Appendix~\ref{subsec:appendix_full_exp_res}. Detailed descriptions of all baselines are in Appendix~\ref{subsec:appendix_baselines}.

\update{
\textbf{Implementation.} Similar to standard implementations of DPSGD~\cite{Abadi2016dpsgd}, the boundedness assumption (Assumption~\ref{ass:lipschitzness}) used in our analysis are enforced in practice through clipping, with clipping thresholds treated as hyperparameters. Quantities such as the noise scale and learning rate are calibrated accordingly. In particular, although the gradient sensitivity depends on the bound $M$ in Assumption~\ref{ass:lipschitzness}, which itself depends on $\|\rvx^{\mathrm{pub}}\|$, we do not estimate $M$ directly. Rather, the dependence on $\rvx^{\mathrm{pub}}$ is used only in the analysis to motivate the choice of conditioning matrix $\mC$ and derive utility guarantees.
}

\begin{figure*}[t]
  \centering
  \begin{minipage}[t]{0.75\textwidth}
    \centering
    \includegraphics[width=\linewidth]{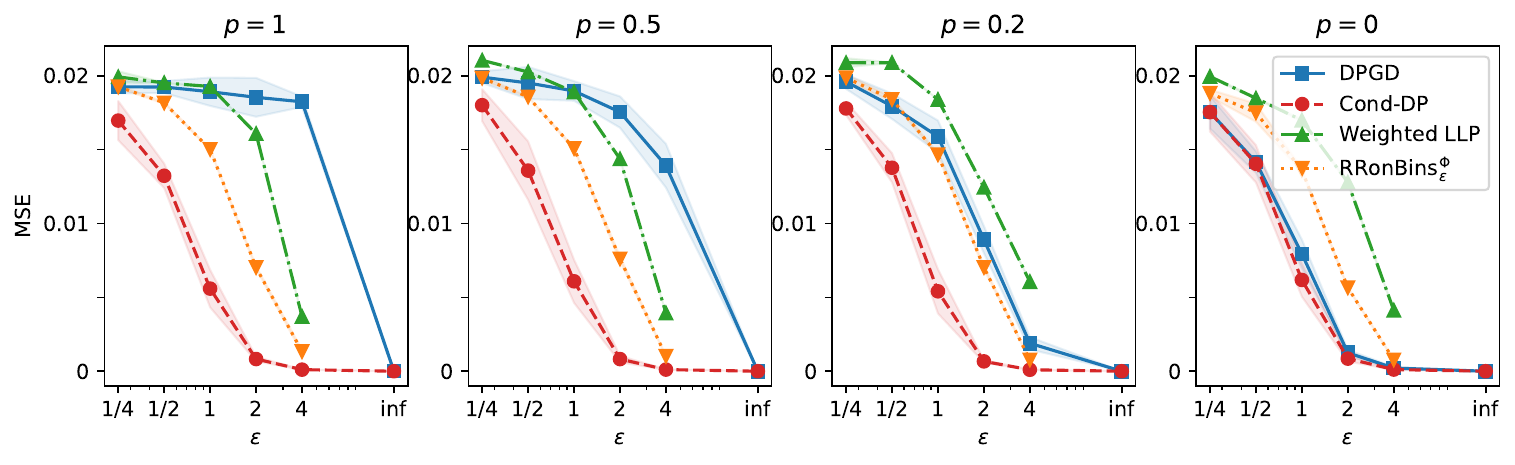}
    \caption{Validation MSEs of different algorithms on synthetic data with varying spectrum power $p$.}
    \label{fig:synthetic_mse}
  \end{minipage}\hfill
  \begin{minipage}[t]{0.2\textwidth}
  \vspace{-10em}
    \centering
    \includegraphics[width=\linewidth]{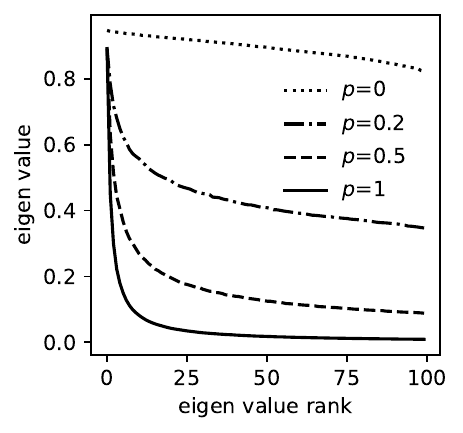}
    \caption{Spectra of singular values in \texttt{synthetic}.}
    \label{fig:spectra_synthetic}
  \end{minipage}
  \vspace{-10pt}
\end{figure*}

\textbf{Hyperparameters.} 
Across all experiments, we evaluate models under $\eps \in \{0.25, 0.5, 1, 2, 4, \infty\}$ (with $\eps = \infty$ denoting non-private training). \update{We fix $\delta = 10^{-6}$ on all datasets except \texttt{Criteo Sponsored Search Conversion} and used $1/(\text{\# train})$ on that dataset instead.}
We extensively tune the following key hyperparameters:
1) \textit{Learning rate} $\eta \in [10^{-5}, 15]$.
2) \textit{Gradient clipping norm} $c \in [0.05, 40]$ which controls the sensitivity of gradients via clipping~\cite{Abadi2016dpsgd}, and
3) \textit{Initialization standard deviation} $\sigma^{\text{init}} \in [10^{-5}, 5]$ used in Gaussian initialization of each model parameter. 
For the optimizer, we use \texttt{Opacus}~\cite{opacus} to implement a noisy version of Adam~\cite{kingma2015adam} with default \texttt{PyTorch} hyperparameters.

\begin{table}[t]
\begin{adjustbox}{width=0.48\textwidth}
    \begin{tabular}{|c|c|c|c|c|}
    \hline
        Dataset & \# train & \# test & $d^{\text{pub}}$ & Model\\
    \hline
        \texttt{synthetic} & 4000 & 1000 & 100 & Linear \\
    \hline
        \texttt{Boston\_housing}
        & 404 & 102 & 12 & Linear \\
        \texttt{wine} & 3918 & 980 & 11 & Linear \\
        \texttt{energy} & 15788 & 3947 & 25 & Linear \\
        \texttt{CA\_housing} & 16512 & 4128 & 8 & Linear \\
    \hline
        \texttt{supercond} & 17010 & 4253 & 81 & MLP \\
        \texttt{TomsHardware} & 22543 & 5636 & 96 & MLP \\
        \texttt{wave} & 28834 & 7209 & 98 & MLP\\
    \hline
        \shortstack{\texttt{Criteo Sponsored} \\ \texttt{Search Conversion}} & 1386177 & 346544 & 4003 & Linear / MLP \\
    \hline
    \end{tabular}
\end{adjustbox}
    \caption{A Summary of datasets.}
    \vspace{-20pt}
    \label{tab:ds_summary}
\end{table}

\textbf{Datasets and Common Setup.}
Table~\ref{tab:ds_summary} summarizes all datasets used, including the number of training and test samples, feature dimensions, and model architectures. 
For each experiment setting, we report the mean (solid lines) and one standard deviation (shaded areas) over 5 runs.

\subsection{Private Linear Regression on Synthetic Datasets}
\label{subsec:private_lin_reg_synthetic_ds}

We first evaluate all algorithms on synthetic data. 
The public feature matrix is generated as $\mX^{\text{pub}} = \mU \Sigma \mV^{\top}$, where $\mU \in \R^{n \times d^{\text{pub}}}$ has orthogonal columns, 
$\mV \in \R^{d^{\text{pub}} \times d^{\text{pub}}}$ is orthonormal, and $\Sigma$ is diagonal with $\Sigma_{ii} = 1/i^p$ for $p \ge 0$. Varying $p$ controls the spectral decay of $\mX^{\text{pub}}$, with larger $p$ corresponding to faster decay. We consider $p \in \{0, \frac{1}{5}, \frac{1}{2}, 1\}$ and train all models for 1024 epochs using full-batch optimization. Figure~\ref{fig:spectra_synthetic} shows the resulting singular value spectra\footnote{For $p=0$, eigenvalues are not exactly 1 because the full dataset ($n=5000$) is generated jointly, while the plot reflects the training subset ($n=4000$).}, and Figure~\ref{fig:synthetic_mse} reports validation losses under varying privacy budgets $\epsilon$.

\textbf{Discussion.}
The proposed \precond\ consistently achieves the lowest validation loss among all algorithms, including the state-of-the-art \rronbins, across different privacy budgets $\eps$, supporting our theoretical findings.
As the power $p$ increases and the spectrum decays more rapidly, the performance gap between \precond\ and other baselines widens, indicating that \precond\ effectively leverages spectral information and excels in ill-conditioned settings.
When $p=0$, where the singular values are uniform, \precond\ performs similarly to \dpsgd, further validating our intuition.

\subsection{Private Linear Regression on Real-World Datasets}
\label{subsec:private_lin_reg_real_world_ds}

Moving beyond synthetic data, we train private linear models on four real-world datasets: \texttt{wine}, \texttt{Boston\_housing}, \texttt{energy}, and \texttt{CA\_housing}. Figure~\ref{fig:spectrum_datasets} presents the singular value spectra for two of these datasets, with the remaining spectra reported in Appendix~\ref{subsec:appendix_dataset_spectra}. Due to the relatively small dataset sizes, all models are trained for 128 epochs using full-batch optimization. The corresponding results are presented in Figure~\ref{fig:mse_linear_model_real_datasets}.

\begin{figure}[h]
    \centering
    \includegraphics[width=\linewidth]{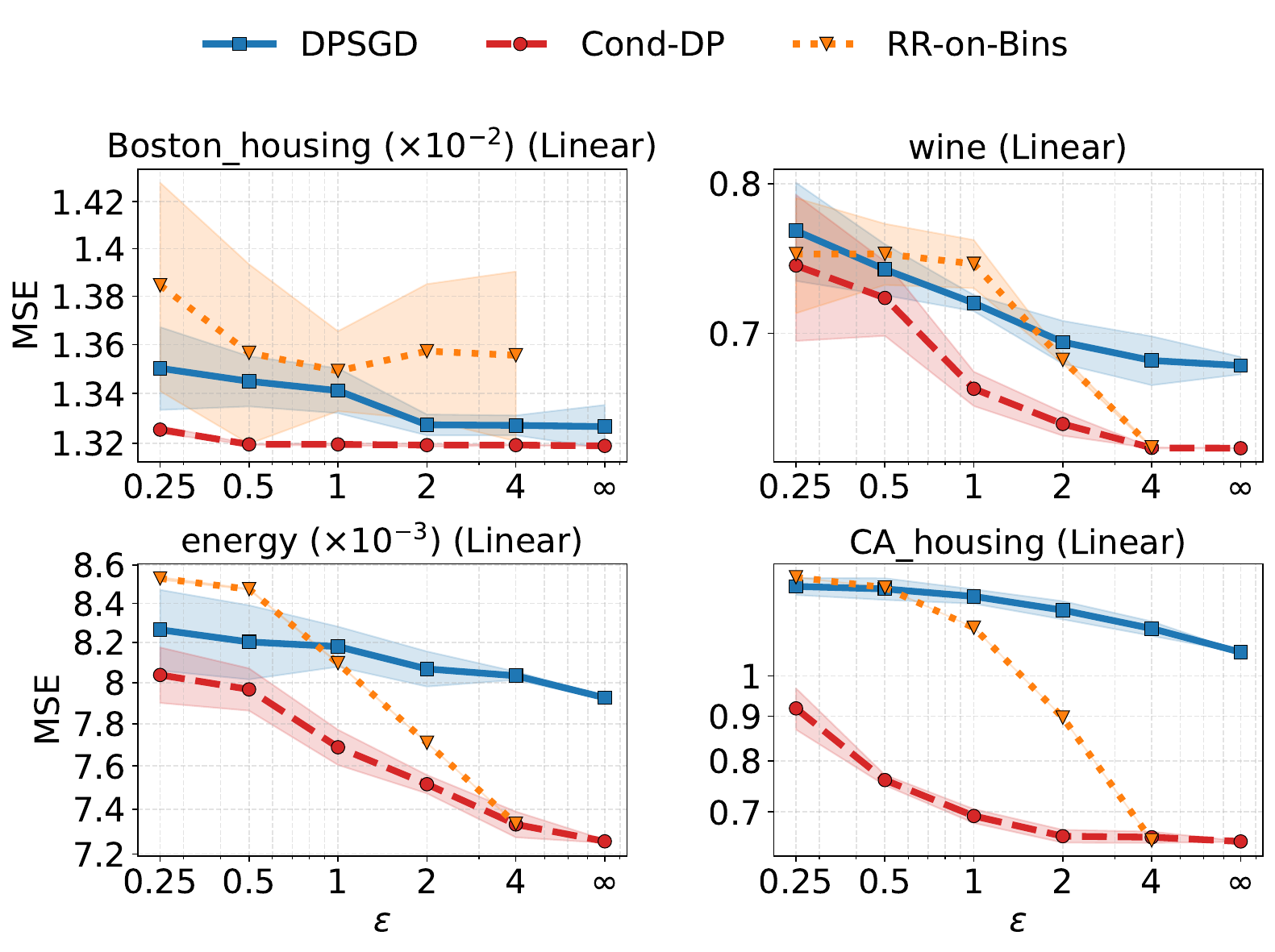}
    \caption{Validation MSEs of different algorithms for private linear models on real-world datasets across varying privacy budgets $\epsilon$.}
    \vspace{-15pt}
    \label{fig:mse_linear_model_real_datasets}
\end{figure}

\textbf{Discussion.} 
On real-world datasets with rapidly decaying feature spectra, we observe that \precond\ consistently outperforms other methods in private linear regression, corroborating our theory.

\begin{figure*}[t]
    \centering
    \includegraphics[width=0.6\linewidth]{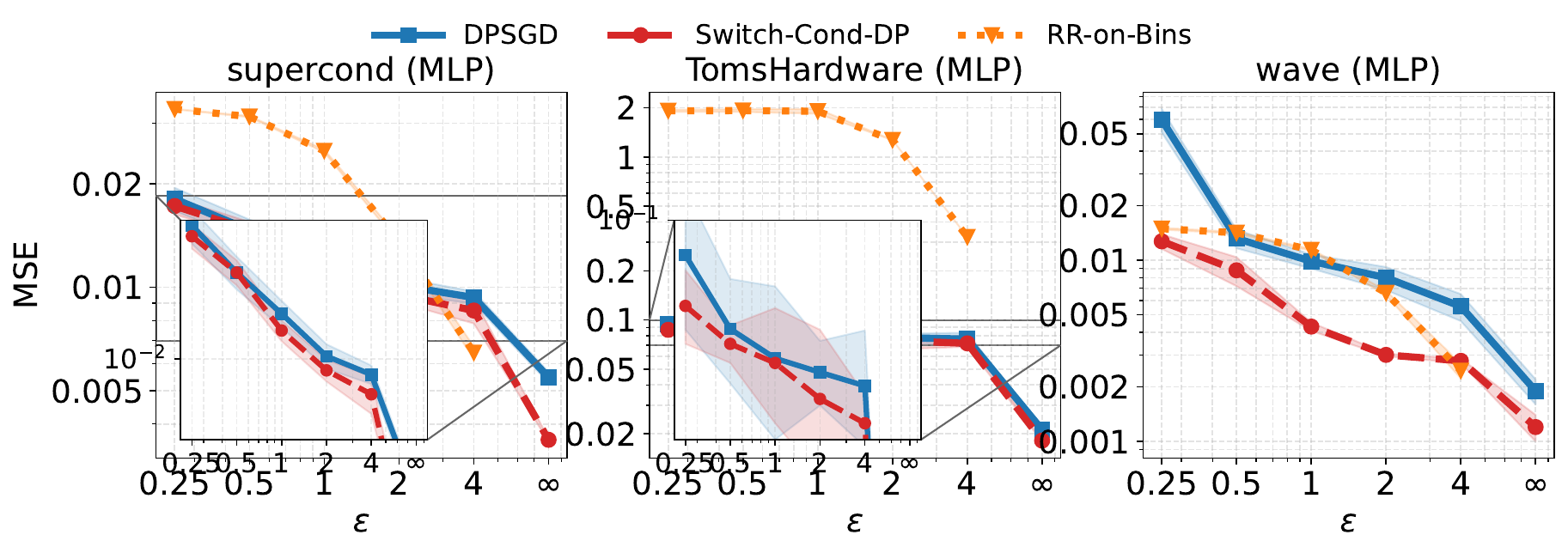}
    \vspace{-5pt}
    \caption{Validation MSEs of different algorithms for private \update{linear models with linear input transformations and} a two-layer MLP prediction layer (16, 8 hidden units) on real-world datasets across varying privacy budgets 
    $\epsilon$.
    }
    \vspace{-10pt}
    \label{fig:mse_glm_real_datasets}
\end{figure*}

\subsection{\update{Private Two-Layer MLP Regression}}
\label{subsec:general_glms}

We extend our evaluation beyond private linear regression by using a model with
$p^{\text{pub}} = 16$, and the prediction component $f_{\omega}$ on top of the linear embedding layers is a 2-layer MLP with 16 and 8 hidden units, followed by a scalar output.
We evaluate this model trained by different algorithms on three datasets: \texttt{supercond}, \texttt{TomsHardware} and \texttt{wave}.
Owing to the larger dataset sizes, training is performed for 1000 epochs using batch size 100.
The spectra of the datasets can be found in Appendix~\ref{subsec:appendix_dataset_spectra}. The results are presented in Figure~\ref{fig:mse_glm_real_datasets}.

\textbf{Discussion.}
Across all datasets with decaying spectra, \switch\ consistently achieves lower MSE than \dpsgd, and substantially outperforms \rronbins\ in high-privacy regimes. These results indicate that the benefits of \precond\ extend beyond linear regression, where they are theoretically justified, to more complex model architectures.

\subsection{Private Linear and \update{Two-layer MLP} Regression on Criteo Sponsored Search Conversion Log Dataset}
\label{subsec:criteo_search}

\texttt{Criteo Sponsored Search Conversion}
\footnote{\url{https://ailab.criteo.com/criteo-sponsored-search-conversion-log-dataset/}}
is a widely used benchmark for advertising prediction tasks. We evaluate both private linear models and private \update{models with linear input transformations and} MLP-based prediction layers on this dataset.

\update{
\textbf{Data Preprocessing.} Following~\cite{ghazi2023regression_ldp}, we remove all examples with no recorded conversion (i.e., \texttt{SalesAmountInEuro} = -1), yielding a dataset of 1,732,721 examples, and clip the conversion value at 400 Euro, which corresponds to the
95th percentile of the value distribution. We randomly split the data into 80\% for training and 20\% for testing.
We retain three numerical features (time\_delay\_for\_conversion, nb\_clicks\_1week, product\_price) and process the 17 categorical features by hashing each feature–value pair into one of 4,000 bins, resulting in a total of 4,003 input features. For private linear regression, we directly use these 4,003 features. For private models with an MLP prediction layer, we learn an embedding of dimension 8 for each hashed categorical bin, and concatenate the resulting categorical embeddings, along with the three numerical features. This yields an embedding of size $17\times 8 + 3$ for each sample.
}

\textbf{Adapting to Categorical Features.}
Due to the natural sparsity of categorical features, directly conditioning the input vector would destroy sparsity, making embedding+concatenation no longer possible. To extend our approach to this sparse setting, we adopt an analogous construction to $\widehat{\mC}$, but use it as an optional \emph{initialization}, rather than a conditioning factor. Specifically, we compute the singular value decomposition of $\mX^{\text{pub}} = \mU \Sigma \mV^{T}$, and use $\Sigma^{-1} \mV^{T}_{:d^{\text{emb}}}$ —the top $d^{\text{emb}}$ right singular vectors scaled by the inverse singular values—to initialize the learnable embedding table. The intuition is that with such initialization, the outputs of the embedding layer would be approximately spherical, potentially facilitating learning during early phases of training.

\textbf{Setup.}
Following~\cite{ghazi2023regression_ldp}, for \update{non-linear models},
we use a two-layer MLP with 128, 64 hidden units in the prediction layer, followed by a scalar output.
Both linear and \update{the two-layer MLP} models are trained for 50 epochs using a batch size of 8,192.

\begin{figure}
    \centering
    \vspace{-5pt}
    \includegraphics[width=0.48\linewidth]{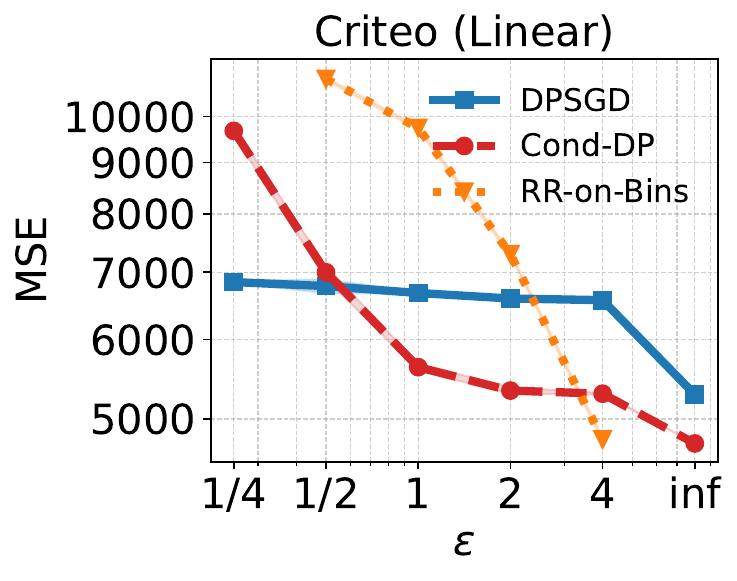}
    \includegraphics[width=0.48\linewidth]{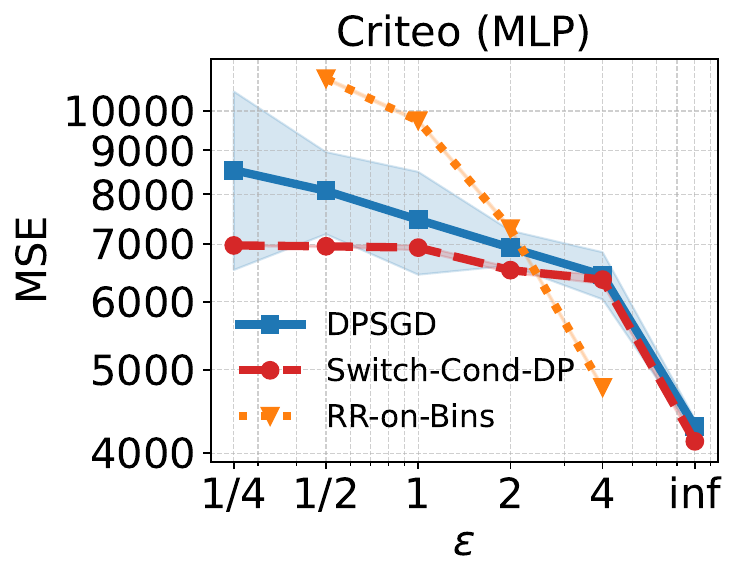}
    \caption{Validation MSEs of different algorithms on \texttt{Criteo Sponsored Search Conversion} across varying privacy budgets $\epsilon$. (left) linear models; (right) \update{models with linear input transformations and} a two-layer MLP prediction layer (128, 64 hidden units).}
    \vspace{-20pt}
    \label{fig:mse_criteo_search}
\end{figure}

\textbf{Results.}
\rronbins\ incurs a computational cost of $O((n^{\text{pub}})^2)$, which scales poorly with dataset size, making it difficult to run on this dataset containing over 1.3 million training examples, even when using label discretization techniques proposed in~\cite{ghazi2023regression_ldp}. We therefore report the results for \rronbins\ as published in~\cite{ghazi2023regression_ldp} and compare against these numbers in both the linear and \update{non-linear} settings.
In non-private settings, it is well known that nonlinear models, such as \update{models} with MLP-based prediction layers, substantially outperform linear models on this dataset~\cite{rendle2010FM}, and are therefore the standard choice. Figure~\ref{fig:mse_criteo_search} reports the validation MSEs for private linear models (left) and private MLP (right).

\textbf{Discussion.}
Interestingly, we observe that linear models can be competitive—and in some regimes superior—in private settings on this dataset, particularly under high privacy regimes. For linear models, \precond\ consistently improves upon \dpsgd\ in low-to-moderate privacy regimes, with the largest gains observed for $\eps \in (0.5, 2]$, and substantially outperforms \rronbins\ (the previous SOTA on this benchmark) across most privacy budgets (except $\epsilon = 4$).
For the more expressive \update{non-linear} setting with a two-layer MLP prediction layer, the results indicate that \switch\ achieves lower MSE than both \dpsgd\ and \rronbins\ for small privacy budgets $(\epsilon < 4)$.
\vspace{-5pt}

\section{Conclusion}
\label{sec:conclusion}

In this work, we develop new methods for leveraging non-sensitive public features in differentially private learning, focusing on underexplored regression settings. We propose \precond, which use conditioning matrices derived from public features to improve privacy–utility trade-offs, and provide convergence guarantees for convex, strongly convex, and non-convex objectives. We show that \precond\ provably outperforms \dpsgd\ in private linear regression and extend these gains to nonlinear models via \switch. Experiments on diverse synthetic and real-world datasets confirm consistent improvements over state-of-the-art baselines.

\section*{Acknowledgements}
\update{We thank the reviewers and area chair for the insightful discussion and suggestions.}

\section*{Impact Statement}
This paper presents work whose goal is to advance the field of Machine
Learning. There are many potential societal consequences of our work, none
which we feel must be specifically highlighted here.

\bibliography{mybib}
\bibliographystyle{icml2026}

\newpage
\appendix
\onecolumn
\section{Privacy Analysis}
\label{sec:appendix_privacy}

\begin{theorem}[Privacy guarantee, Restatement of Theorem~\ref{thm:privacy}]
\label{thm:appendix_privacy}
    If the noise variance in Algorithm ~\ref{alg:cond_dp} is set to be $\sigma^2 = \widetilde{O}(\frac{M^2 T}{\eps^2 n^2})$, where $M^2 \triangleq G^2 \cdot
    \max_{i\in [n]} \|\mC\rvx_i^{\text{pub}}\|^2 
    + \widehat{G}^{2} \cdot \update{R^2}
    + \overline{G}^2$, then Algorithm~\ref{alg:cond_dp} is $(\eps, \delta)$-differentially private.
\end{theorem}

\begin{proof}
    The privacy analysis follows directly from the Gaussian mechanism. To ensure that the final output satisfies $(\eps, \delta)$-DP, it suffices by standard composition~\cite{dwork2014dp} that each of the $T$ iterations incurs a privacy loss of 
    $(\gO(\frac{\eps}{\sqrt{T}}), \gO(\frac{\delta}{T}))$. 
    By the Gaussian mechanism~\cite{dwork2014dp}, this is achieved by adding noise with variance $\sigma^2 = \widetilde{\gO}(\frac{ S \sqrt{T}}{\eps})$,
    where $S$ denotes the $\ell_2$ sensitivity of the per-sample gradient. 
    The $\ell_2$ sensitivity of the per-sample gradient can be characterized by $M^2 \triangleq G^2 \cdot \max_{i\in [n]} \|\mC\rvx_i^{\text{pub}}\|^2 
    + \widehat{G}^{2} \cdot \update{R^2}
    + \overline{G}^2$.
    Since $S \leq \frac{M}{n}$, it suffices to set $\sigma^2 = \widetilde{O}(\frac{S^2 T}{\eps^2}) = \widetilde{\gO}(\frac{M^2 T}{n^2 \eps^2})$.
\end{proof}

\section{Convergence of Private \update{Models with Linear Input Transformations} 
Under Convex Losses}
\label{sec:appendix_convergence_glm}

\begin{theorem}[Restatement of Theorem~\ref{thm:convergence_glm}]
\label{thm:appendix_convergence_glm}
    Let $\mC$ be a conditioning matrix of full rank.  Under Assumptions~\ref{ass:loss_glm} and~\ref{ass:lipschitzness},
    if Algorithm~\ref{alg:cond_dp} is run with
    \begin{align}
        &\eta = \frac{N}{\sqrt{TM^2(1 + \frac{dT}{n^2\eps^2})}}, \quad \sigma^2 = \widetilde{O}(\frac{M^2 T}{n^2\eps^2})\\
        &\text{where } M^2 \triangleq G^2 \cdot \max_{i\in [n]} \|\mC\rvx_i^{\text{pub}}\|^2 
        + \widehat{G}^{2} \cdot \update{R^2} 
        + \overline{G}^2, \\
        \nonumber
        &\text{and } N^2 \triangleq \text{Tr}\Big( (\Theta_0^{\text{pub}} \mC - (\Theta^{\text{pub}})^{*})  (\mC^T \mC)^{-1}(\Theta_0^{\text{pub}} \mC - (\Theta^{\text{pub}})^{*})^{T} \Big) + \|\Theta^{\text{priv}}_0 - (\Theta^{\text{priv}})^*\|_F^2
        + \|\omega_0 - \omega^*\|^2
    \end{align}
    then it guarantees
    \begin{align}
        &\E\left[\gL(\Theta^{\text{pub-out}}, \Theta^{\text{priv-out}}, \omega^{\text{out}}; D)\right]
        - \gL^*\\
        \nonumber
        &\leq \sqrt{\frac{1}{T} + \frac{d}{n^2\eps^2} } \cdot \underbrace{ M}_{\text{Gradient bound factor}}
        \cdot \underbrace{N}_{\text{Initialization factor}}
    \end{align}
    where the expectation is taken w.r.t. the randomness of the algorithm. 
\end{theorem}

\begin{proof}
    For notational simplicity, we represent all model parameters as vectors. Specifically, let $\theta^{\text{pub}, j}$ denote the $j$-th row of $\Theta^{\text{pub}} \in \R^{p^{\text{pub}} \times d^{\text{pub}}}$ for each $j \in [p^{\text{pub}}]$, and similarly, let $\theta^{\text{priv}, j}$ denote the $j$-th row of $\Theta^{\text{priv}} \in \R^{p^{\text{priv}} \times d^{\text{priv}}}$ for each $j \in [p^{\text{priv}}]$.

    Let the vector of parameters be $\theta_t \in \R^{d}$, $\forall t\in [T]$, and the corresponding vector of parameters with conditioning matrix $\mC \in \R^{\text{pub}\times \text{pub}}$ be $\mW \theta_t \in \R^{d}$, for $\mW \in \R^{d\times d}$, then
    \begin{align}
    \label{eq:def_vector_param}
        \theta_t = \begin{bmatrix}
        \theta_t^{\text{pub}, 1}\\
        \theta_t^{\text{pub}, 2}\\
        \dots\\
        \theta_t^{\text{pub}, p^{\text{pub}}}\\
        \theta_t^{\text{priv}, 1}\\
        \dots\\
        \theta_t^{\text{priv}, p^{\text{priv}}}\\
        \omega_t
    \end{bmatrix},
    \theta^{\text{C}}  = \begin{bmatrix}
        \mC^T \theta_t^{\text{pub}, 1}\\
        \mC^T \theta_t^{\text{pub}, 2}\\
        \dots\\
        \mC^T \theta_t^{\text{pub}, p^{\text{pub}}}\\
        \theta_t^{\text{priv}, 1}\\
        \dots\\
        \theta_t^{\text{priv}, p^{\text{priv}}}\\
        \omega_t
    \end{bmatrix} = \mW \theta_t, \quad 
    \text{where } \mW = \begin{bmatrix}
        \mC^{T} & \mathbf{0} & \dots & \mathbf{0} & \dots & \mathbf{0} & \mathbf{0}\\
        \mathbf{0} & \mC^{T} & \dots & \mathbf{0} & \dots & \mathbf{0} & \mathbf{0} \\
        \mathbf{0} & \cdots & \mC^{T} & \mathbf{0} & \dots & \mathbf{0} & \mathbf{0}\\
        \mathbf{0} & \mathbf{0} & \dots & \sI_{d^{\text{priv}}} & \dots & \mathbf{0} & \mathbf{0}\\
        \dots\\
        \mathbf{0} & \mathbf{0} & \dots & \mathbf{0} & \dots & \sI_{d^{\text{priv}}} & \mathbf{0}\\
        \mathbf{0} & \mathbf{0} & \dots & \mathbf{0} & \dots & \mathbf{0} & \sI_{p}\\
    \end{bmatrix}
    \end{align}
    and the inverse map is $\mW^{-1}$ such that
    \begin{align}
    \label{eq:def_W_dag}
        \mW^{-1} = \begin{bmatrix}
        (\mC^{T})^{-1} & \mathbf{0} & \dots & \mathbf{0} & \dots & \mathbf{0} & \mathbf{0}\\
        \mathbf{0} & (\mC^{T})^{-1} & \dots & \mathbf{0} & \dots & \mathbf{0} & \mathbf{0} \\
        \mathbf{0} & \cdots & (\mC^{T})^{-1} & \mathbf{0} & \dots & \mathbf{0} & \mathbf{0}\\
        \mathbf{0} & \mathbf{0} & \dots & \sI_{d^{\text{priv}}} & \dots & \mathbf{0} & \mathbf{0}\\
        \dots\\
        \mathbf{0} & \mathbf{0} & \dots & \mathbf{0} & \dots & \sI_{d^{\text{priv}}} & \mathbf{0}\\
        \mathbf{0} & \mathbf{0} & \dots & \mathbf{0} & \dots & \mathbf{0} & \sI_{p}\\
    \end{bmatrix}
    \end{align}

    Let the vector version of the output of Algorithm~\ref{alg:cond_dp}, representing $(\Theta^{\text{pub-out}}, \Theta^{\text{priv-out}}, \omega^{\text{out}})$, be 
    \begin{align}
    \label{eq:def_theta_C_out}
        \theta^{\text{out}} = \frac{1}{T}\sum_{t=1}^{T} \theta_t
    \end{align}
    Let $\theta^*$ denote the optimal parameter vector, corresponding to $((\Theta^{\text{pub}})^{*}, (\Theta^{\text{priv}})^{*}, \omega^{*})$.
    Let $\rvb_t$ and $\rvg_t$ be the noise and gradient vector at iteration $t\in [T]$, corresponding to $(\mB_t^{\text{pub}}, \mB_t^{\text{priv}}, \rvb_t^{\text{top}})$ and $(\mG_t^{\text{pub}}, \mG_t^{\text{priv}}, \rvg_t^{\text{top}})$, respectively.

    We note that 
    \begin{align}
        \rvg_t 
        &= \left[ \left( \frac{\partial \gL^{(t)}}{\partial \theta_t^{\text{pub}, 1}}\right)^T, \dots, \left( \frac{\partial \gL^{(t)}}{\partial \theta_t^{\text{pub}, p^{\text{pub}}}} \right)^T,
        \Big( \frac{\partial \gL^{(t)}}{\partial \theta_t^{\text{priv}, 1}} \Big)^{T}, \dots,
        \Big( \frac{\partial \gL^{(t)}}{\partial \theta_t^{\text{priv}, p^{\text{priv}}}} \Big)^{T},
        \Big( \frac{\partial \gL^{(t)}}{\partial \omega_t}\Big)^{T}
        \right]^T\\
        &= \mW^T \nabla \gL(\mW \theta_t;D)
    \end{align}
    where 
    \begin{align*}
    \label{eq:def_loss_grad_details}
        &\frac{\partial \gL^{(t)}}{\partial \theta_t^{\text{pub}, j}} 
        = \frac{1}{n}\sum_{i=1}^{n} \frac{\partial l_{t, i}}{\partial \langle \mC^{T}\theta_t^{\text{pub}, j} , \rvx_i^{\text{pub}} \rangle} \cdot \mC \rvx_i^{\text{pub}}, \forall j\in [p^{\text{pub}}]\\
        &\frac{\partial \gL^{(t)}}{\partial \theta_t^{\text{priv}, j}} 
        = \frac{1}{n}\sum_{i=1}^{n} \frac{\partial l_{t, i}}{\partial \theta_t^{\text{priv}, j}}
        = \frac{1}{n}\sum_{i=1}^{n} \frac{\partial l_{t, i}}{\partial \langle \theta_t^{\text{priv}, j} , \rvx_i^{\text{priv}}\rangle}\cdot \rvx_i^{\text{priv}},
        \forall j\in [p^{\text{priv}}]\\
        &\frac{\partial \gL^{(t)}}{\partial \omega} = \frac{1}{n}\sum_{i=1}^{n}\frac{\partial l_{t,i}}{\partial \omega}
    \end{align*}
    and recall that $l_{t, i} = \nabla \gL_i(\Theta_t^{\text{pub}}\mC, \Theta_t^{\text{priv}}, \omega_t), \forall i\in [n], t\in[T]$.
    
    By convexity of $\gL$ (Assumption~\ref{ass:loss_glm}), 
    \begin{align}
        \gL(\mW \theta^{\text{out}}; D) - \gL(\theta^{*}; D)
        &= \gL(\frac{1}{T}\sum_{t=1}^{T} \mW \theta_t; D) - \gL(\theta^{*}; D)
        \leq \frac{1}{T}\sum_{t=1}^{T}\Big( \gL(\mW \theta_t; D) - \gL(\theta^{*}; D) \Big)\\
    \label{eq:convergence_skeleton}
        &\leq \frac{1}{T}\sum_{t=1}^{T} \langle \nabla \gL(\mW \theta_t; D), \mW\theta_t - \theta^{*} \rangle
    \end{align}

    In the following, we use $\E_{\gE}[\gA]$ to denote the conditional expectation of variable $\gA$ conditioned on the event $\gE$. Let $\Phi_t(\theta) = \E_{\rvb_1, \dots, \rvb_t}\left[ \left\| \mW \theta - \theta^* \right\|_{(\mW \mW^{T})^{-1}}^2 \right]$ where recall that $\|\cdot\|_A$ denotes the norm $\|\theta\|_A^2 = \langle \theta , A \theta \rangle$.
    
    By the update rule in Algorithm~\ref{alg:cond_dp},
    \begin{align}
        \Phi_{t}(\theta_{t+1})
        &= \E_{\rvb_1,\dots, \rvb_t}\left[ \|\mW \theta_{t+1} - \theta^{*} \|_{(\mW \mW^T)^{-1}}^2 \right]\\
        &\leq \E_{\rvb_1,\dots, \rvb_t}\left[ \| \mW \theta_t - \theta^{*}
        - \eta \mW \Big( \rvg_t + \rvb_t \Big)\|_{(\mW \mW^{T})^{-1}}^2\right]\\
        &= \E_{\rvb_1,\dots,\rvb_t}\left[ \| \mW \theta_t - \theta^{*} \|_{(\mW\mW^T)^{-1}}^2 \right] \\
        \nonumber
        &\quad - 2\eta \E_{\rvb_1,\dots,\rvb_t} \left[ (\mW (\rvg_t + \rvb_t))^{T} (\mW\mW^{T})^{-1}(\mW \theta_t - \theta^{*})\right]\\
        \nonumber
        &\quad + \eta^2 \E_{\rvb_1,\dots, \rvb_t}\left[\| \mW \Big( \rvg_t + \rvb_t \Big)\|_{(\mW\mW^{T})^{-1}}^2\right]\\
        &= \E_{\rvb_1,\dots,\rvb_t}\left[ \| \mW \theta_t - \theta^{*} \|_{(\mW \mW^{T})^{-1}}^2 \right] \\
        \nonumber
        &\quad - 2\eta \E_{\rvb_1,\dots,\rvb_t} \left[ (\mW \mW^T \nabla \gL(\mW \theta_t ; D))^{T} (\mW \mW^{T})^{-1} (\mW \theta_t - \theta^{*})\right]\\
        \nonumber
        &\quad + \eta^2 \E_{\rvb_1,\dots,\rvb_t}\left[\| \mW \Big( \mW^{T}\nabla \gL(\mW\theta_t; D) + \rvb_t \Big)\|_{(\mW \mW^{T})^{-1}}^2\right]\\
    \label{eq:convergence_glm_interm}
        &= \E_{\rvb_1,\dots, \rvb_{t-1}}\left[ \| \mW \theta_t - \theta^{*} \|_{(\mW \mW^{T})^{-1}}^2 \right]
        - 2\eta \E_{\rvb_1,\dots, \rvb_{t-1}}\left[ \langle \nabla \gL(\mW \theta_t; D), \mW \theta_t - \theta^{*} \rangle \right]\\
        \nonumber
        &\quad + \eta^2\E_{\rvb_1,\dots,\rvb_t}\left[\|\mW \mW^{T}\nabla \gL(\mW\theta_t; D)\|_{(\mW\mW^{T})^{-1}}^2\right] + \eta^2 \E_{\rvb_1,\dots, \rvb_t}\left[\|\mW \rvb_t  \|_{(\mW \mW^{T})^{-1}}^2 \right]\\
        \nonumber
        &= (*)
    \end{align}
    We proceed by bounding the last two variance terms. First, note that
    \begin{align}
    \label{eq:non_convex_ref_2_start}
        &\|\mW \mW^T \nabla \gL(\mW \theta_t; D)\|_{(\mW \mW^{T})^{-1}}^2\\
        \nonumber
        &= (\mW \mW^{T} \nabla \gL(\mW \theta_t; D))^{T} (\mW \mW^{T})^{-1} (\mW \mW^{T} \nabla \gL(\mW \theta_t; D))\\
        &= (\nabla \gL(\mW \theta_t; D))^{T} \mW \mW^{T} \nabla \gL(\mW \theta_t; D)\\
        &= \sum_{j=1}^{p^{\text{pub}}} \left[\frac{\partial \gL^{(t)}}{\partial \mC^T \theta_t^{\text{pub}, j}} \right]^{T} \mC^{T} \mC \left[\frac{\partial \gL^{(t)}}{\partial \mC^T \theta_t^{\text{pub}, j}} \right]
        + \sum_{j=1}^{p^{\text{priv}}}
        \left[\frac{\partial \gL^{(t)}}{\partial \theta_t^{\text{priv}, j}} \right]^{T}
        \left[\frac{\partial \gL^{(t)}}{\partial \theta_t^{\text{priv}, j}} \right]
        + \left[\frac{\partial \gL^{(t)}}{\partial \omega}\right]^{T}
        \left[\frac{\partial \gL^{(t)}}{\partial \omega}\right]\\
        &= \sum_{j=1}^{p^{\text{pub}}} \left\| \mC \frac{\partial \gL^{(t)}}{\partial \mC^T \theta_t^{\text{pub}, j}} \right\|^2
        + \sum_{j=1}^{p^{\text{priv}}} \left\| \frac{\partial \gL^{(t)}}{\partial \theta_t^{\text{priv}, j}} \right\|^2
        + \left\| \frac{\partial \gL^{(t)}}{\partial \omega} \right\|^2\\
        &= \sum_{j=1}^{p^{\text{pub}}} \left\| \mC \frac{1}{n}\sum_{i=1}^{n} \frac{\partial l_{t, i}}{\partial \langle \mC^{T}\theta_t^{\text{pub}, j} , \rvx_i^{\text{pub}} \rangle} \cdot \rvx_i^{\text{pub}} \right\|^2
        + \sum_{j=1}^{p^{\text{priv}}} \left\| \frac{1}{n} \sum_{i=1}^{n} \frac{\partial l_{t, i}}{\partial \langle \theta_t^{\text{priv}, j}, \rvx_i^{\text{priv}} \rangle} \cdot \rvx_i^{\text{priv}}\right\|^2
        + \left\| \frac{\partial \gL^{(t)}}{\partial \omega} \right\|^2
        \\
        &\leq \sum_{j=1}^{p^{\text{pub}}} \left\| \mC \frac{1}{n}\sum_{i=1}^{n} \frac{\partial l_{t, i}}{\partial \langle \mC^{T}\theta_t^{\text{pub}, j} , \rvx_i^{\text{pub}} \rangle} \cdot \rvx_i^{\text{pub}} \right\|^2 
        + \widehat{G}^{2} \cdot \update{R^2}
        + \overline{G}^{2}\\
        \nonumber
        &\quad \text{By Assumption~\ref{ass:lipschitzness}}\\
        &\leq \sum_{j=1}^{p^{\text{pub}}} \Big( \frac{1}{n}\sum_{i=1}^{n} \frac{\partial l_{t, i}}{\partial \langle \mC^{T}\theta_t^{\text{pub}, j} , \rvx_i^{\text{pub}} \rangle}\Big)^2 \cdot \max_{i\in [n]} \left\| \mC \rvx_i^{\text{pub}}\right\|^2 
         + \widehat{G}^{2} \cdot \update{R^2}
         + \overline{G}^{2}\\
        &\leq G^{2} \cdot \max_{i\in [n]}\left\| \mC \rvx_i^{\text{pub}}\right\|^2 
        + \widehat{G}^{2} \cdot \update{R^2}
        + \overline{G}^2 \label{eq:grad_norm_bound}
    \end{align}
where the last inequality is by Assumption~\ref{ass:lipschitzness}.
   
   Moreover, let $\rho^{(p)} \sim \gN(0, \sigma^2 \sI_{p})$ denote a random Gaussian vector of size $p$,
   \begin{align}
   \label{eq:non_convex_ref_1_start}
        &\E_{\rvb_1,\dots,\rvb_t}\left[\| \mW \rvb_t \|_{(\mW \mW^{T})^{-1}}^2\right]
        = \E_{\rvb_1,\dots,\rvb_t}\left[ \rvb_t^{T} \mW^{T} (\mW \mW^{T})^{-1}\mW \rvb_t \right]\\
        &= \sum_{j=1}^{p^{\text{pub}}} \E_{\rvb_1,\dots,\rvb_t}\left[ \rho_j^{(d^{\text{pub}})}\mC (\mC^{T}\mC)^{-1} \mC^T \rho_j^{(d^{\text{pub}})}  \right]
        + (p^{\text{priv}} d^{\text{priv}} + p) \sigma^{2}\\
        &= \sum_{j=1}^{p^{\text{pub}}} \sigma^{2} \text{Tr}(\mC (\mC^T \mC)^{-1}\mC^{T}) + (p^{\text{priv}} d^{\text{priv}} + p) \sigma^{2}\\
        &=  \sum_{j=1}^{p^{\text{pub}}} \sigma^{2} \text{Tr}((\mC^{T}\mC)^{-1} \mC^{T}\mC) + (p^{\text{priv}}d^{\text{priv}} + p) \sigma^{2}
        &\text{$\because$ Cyclic rule of trace}\\
    \label{eq:non_convex_ref_1_end}
        &= (p^{\text{pub}} d^{\text{pub}} + p^{\text{priv}} d^{\text{priv}} + p) \sigma^2 = d \sigma^2
   \end{align}

    Hence, following (*) in Eq.~\ref{eq:convergence_glm_interm},
    \begin{align}
    \label{eq:convergence_interm}
        \Phi_{t}(\theta_{t+1})
        &\leq \Phi_{t-1}(\theta_t) 
        - 2\eta \E_{\rvb_1,\dots, \rvb_{t-1}}\left[ \langle \nabla \gL(\mW \theta_t; D), \mW \theta_t - \theta^{*} \rangle \right]\\
        \nonumber
        &\quad + \eta^2\Big( G^{2} \cdot \max_{i\in [n]}\left\| \mC \rvx_i^{\text{pub}}\right\|^2 
         + \widehat{G}^{2} \cdot \update{R^2} + \overline{G}^2 + d\sigma^{2} \Big)
    \end{align}

    and therefore,
    \begin{align}
    \label{eq:recurrence}
        \E_{\rvb_1,\dots,\rvb_t}\left[\langle \nabla \gL(\mW \theta_t; D), \mW \theta_t - \theta^{*}\rangle \right]
        &\leq \frac{1}{2\eta}\Big( \Phi_{t-1}(\theta_t) - \Phi_{t}(\theta_{t+1})\Big)\\
        \nonumber
        &\quad + \frac{\eta}{2} \Big( G^{2} \cdot \max_{i\in [n]}\left\| \mC \rvx_i^{\text{pub}}\right\|^2 
        \update{+\widehat{G}^2\cdot R^2}
        + \overline{G}^2 + d\sigma^{2} \Big)
    \end{align}

    Summing from $t=1$ to $T$, by Eq.~\ref{eq:recurrence} and Eq.~\ref{eq:convergence_skeleton}, we have
    \begin{align}
        &\E\left[\gL(\mW \theta^{\text{out}}; D) \right] - \gL(\theta^{*}; D)
        \leq \frac{1}{T}\sum_{t=1}^{T}
        \E_{\rvb_1,\dots,\rvb_{t-1}}\left[\langle \nabla \gL(\mW\theta_t; D), \mW \theta_t - \theta^{*} \rangle\right]\\
        &\leq \frac{1}{2\eta T} \| \mW \theta_0 - \theta^{*} \|_{(\mW \mW^{T})^{-1}}^2
        + \frac{\eta}{2} \Big( G^{2} \cdot \max_{i\in [n]}\left\| \mC \rvx_i^{\text{pub}}\right\|^2 
        \update{+ \widehat{G}^2 \cdot R^2}
        + \overline{G}^2 + d\sigma^{2} \Big) \label{eq:convergence_glm_raw}
    \end{align}

    Setting the noise variance as $\sigma^2 = \widetilde{\gO}(\frac{M^2 T}{n^2 \eps^2})$, and following Eq.~\ref{eq:convergence_glm_raw}, we have
    \begin{align}
    \label{eq:convergence_glm_raw_2}
        &\E\left[\gL(\mW \theta^{\text{out}}; D) \right] - \gL(\theta^{*}; D)
        \leq \frac{1}{2\eta T} \| \mW \theta_0 - \theta^{*} \|_{(\mW \mW^{T})^{-1}}^2 
        + \frac{\eta}{2} \Big( M^2 + \frac{d M^2 T}{n^2 \eps^2}
        \Big)
    \end{align}

    \textbf{Setting the learning rate $\eta$.}
    To minimize the R.H.S. of Eq.~\ref{eq:convergence_glm_raw_2}, let
    \begin{align}
        &\frac{1}{2T \eta} \| \mW \theta_0 - \theta^{*} \|_{(\mW \mW^{T})^{-1}}^2 
        = \frac{\eta}{2} \Big( M^2 + \frac{d M^2 T}{n^2 \eps^2} \Big)\\
        &\eta = \frac{\| \mW \theta_0 - \theta^{*}\|_{(\mW \mW^{T})^{-1}}}{\sqrt{T M^2 (1 + \frac{dT}{n^2\eps^2} )}}
    \end{align}
    and hence,
    \begin{align}
        &\E\left[\gL(\mW \theta^{\text{out}}; D) \right] - \gL(\theta^{*}; D)
        \leq \| \mW \theta_0 - \theta^{*}\|_{(\mW \mW^{T})^{-1}}
        M\sqrt{\frac{1}{T} + \frac{d}{n^2\eps^2}}
    \end{align}
    
    Converting the above bound back to the matrix form of parameters yields the final theorem statement.

\end{proof}

\section{Convergence of Private \update{Models with Linear Input Transformations} Under Strongly Convex Losses}
\label{sec:appendix_convergence_strongly_convex}

The proofs in this section are a generalization of the proofs in~\cite{Rakhlin2012strongly_convex_sco}.

\begin{lemma}[Generalization of Lemma 2 in \cite{Rakhlin2012strongly_convex_sco}]
\label{lemma:strongly_convex_base_case}
    Suppose $\gL(\cdot; D)$ is $\mu$-strongly convex and $\mW$ is a full rank matrix. 
    For any $\theta$, if $\E[\|\nabla \gL(\mW \theta; D)\|^2_{(\mW \mW^{T})^{-1}}] \leq G^2$, then
    \begin{align}
        \E[\|\mW \theta - \theta^{*}\|_{(\mW \mW^T)^{-1}}^2] \leq \frac{4 \lambda_{max}(\mW \mW^{T}) G^2}{\mu^2 \cdot \lambda_{min}(\mW \mW^T)}
    \end{align}
    where $\theta^* = \argmin_{\theta}\gL(\theta; D)$ and $\lambda_{min}, \lambda_{max}$ denote the minimum and maximum eigenvalues respectively. 
\end{lemma}

\begin{proof}
    By strong convexity, for any $\theta$,
    \begin{align}
        \gL(\mW \theta; D) - \gL(\theta^{*}; D) \geq \frac{\mu}{2}\| \mW \theta - \theta^{*}\|^2
    \end{align}
    and by convexity,
    \begin{align}
        \gL(\mW \theta; D) - \gL(\theta^{*}; D) \leq \langle \nabla \gL(\mW \theta; D), \mW \theta - \theta^{*}\rangle
    \end{align}
    Therefore,
    \begin{align}
        \frac{\mu}{2}\|\mW \theta - \theta^{*}\|^2 \leq \langle \nabla \gL(\mW \theta;D), \mW \theta - \theta^{*}\rangle
    \end{align}

    Combining this with the Cauchy-Schwarz inequality, we get,
    \begin{align}
        &\frac{\mu^2}{4} \|\mW \theta - \theta^{*}\|^2 \leq \|\nabla \gL(\mW \theta; D)\|^2\\
        &\|\mW \theta - \theta^{*}\|^2 \leq \frac{4}{\mu^2} \|\nabla \gL(\mW \theta; D)\|^2
    \end{align}
    and so
    \begin{align}
        &\|\mW \theta - \theta^{*}\|_{(\mW \mW^{T})^{-1}}^2
        \leq \lambda_{max}((\mW \mW^{T})^{-1}) \|\mW \theta - \theta^{*}\|^2\\
        &\leq \lambda_{max}((\mW \mW^{T})^{-1})\frac{4}{\mu^2}\|\nabla \gL(\mW \theta;D)\|^2\\
        &\leq \frac{4}{\mu^2}\frac{\lambda_{\max}((\mW \mW^{T})^{-1})}{\lambda_{min}((\mW \mW^{T})^{-1})} \|\nabla \gL(\mW \theta;D)\|_{(\mW \mW^{T})^{-1}}^2\\
        &= \frac{4}{\mu^2}\frac{\lambda_{max}(\mW \mW^{T})}{\lambda_{min}(\mW \mW^{T})} \|\nabla \gL(\mW \theta;D)\|_{(\mW \mW^{T})^{-1}}^2
    \end{align}
    Taking expectation of both sides yields the above lemma statement.    
\end{proof}

\begin{theorem}[Re-statement of Theorem~\ref{thm:convergence_strongly_convex}] Let $\mC$ be a conditioning matrix of full rank.
    Under Assumptions~\ref{ass:loss_strongly_convex} and~\ref{ass:lipschitzness}, if Algorithm~\ref{alg:cond_dp} is run with $\eta_t = \frac{1}{\mu  t}, \forall t$, and noise variance $\sigma^2 = \widetilde{O}(\frac{M^2 T}{\eps^2 n^2 })$, then it guarantees
    \begin{align}
        &\E\left[\gL(\Theta^{\text{pub-out}}\mC, \Theta^{\text{priv-out}}, \omega^{\text{out}} ; D)\right]
        -\gL^{*}
        \leq \frac{16 \beta^2 }{\mu^2}
        (\frac{1}{T} + \frac{d}{n^2\eps^2})
        \cdot K \cdot M^2
    \end{align}
    where 
    \begin{align}
        &K:= (\frac{\sigma_{max}(\mC)}{\sigma_{min}(\mC)})^2 \max\{\sigma_{max}^2(\mC), \sigma^{-2}_{min}(\mC)\}\\
        &M^2 := G^2\cdot \max_{i\in [n]}\|\mC\rvx_i^{\text{pub}}\|^2 + \widehat{G}^2\cdot \update{R^2}
        + \overline{G}^2
    \end{align}
    and $\sigma_{max}(\mC)$ and $\sigma_{min}(\mC)$ denote the maximum and minimum singular value of $\mC$, respectively.
\end{theorem}

\begin{proof}
    For notational simplicity, we use similar terms and definitions as in the proof of the convex case (see Appendix~\ref{sec:appendix_convergence_glm}).
    We represent all model parameters as vectors. 
    The definitions of the flattened parameter vectors $\theta$, $\theta^{C}$, and the transformation matrices $\mW$ and $\mW^{-1}$ follow Eq.~\ref{eq:def_vector_param} and Eq.~\ref{eq:def_W_dag}.

    Following Eq.~\ref{eq:def_theta_C_out}, the vector version of the output of Algorithm~\ref{alg:cond_dp}, representing $(\Theta^{\text{pub-out}}, \Theta^{\text{priv-out}}, \omega^{\text{out}})$, is 
    \begin{align}
        \theta^{\text{out}} = \frac{1}{T}\sum_{t=1}^{T} \theta_t
    \end{align}
    And let the vector version of the optimum be $\theta^*$, representing $ ((\Theta^{\text{pub}})^{*}, (\Theta^{\text{priv}})^{*}, \omega^{*})$.
    and the noise vector at iteration $t\in [T]$ is $\rvb_t$.

    We use $\E_{\gE}[\gA]$ to denote the conditional expectation of variable $\gA$ conditioned on the event $\gE$. The subscript $\gE$ is dropped when the context is clear.
    Let $\Phi_t(\theta) = \E_{\rvb_1, \dots, \rvb_t}\left[ \left\| \mW \theta - \theta^* \right\|_{(\mW \mW^{T})^{-1}}^2 \right]$.
    By the update rule of Algorithm~\ref{alg:cond_dp} and following Eq.~\ref{eq:convergence_interm},
    
    \begin{align}
        \Phi_t(\theta_{t+1})
        &= \E_{\rvb_1,\dots,\rvb_t}\left[
        \| \mW\theta_{t+1} - \theta^{*} \|^2_{(\mW \mW^{T})^{-1}}
        \right]\\
        &\leq \E_{\rvb_1, \dots, \rvb_{t-1}}\left[ \| \mW \theta_t - \theta^{*}\|^2_{(\mW \mW^{T})^{-1}} \right]
        - 2\eta_t \E_{\rvb_1,\dots, \rvb_{t-1}}\left[ \langle \nabla \gL(\mW \theta_t; D), \mW \theta_t - \theta^{*} \rangle \right]\\
        \nonumber
        &\quad + \eta_t^2 \Big( \underbrace{ G^{2} \cdot \max_{i\in [n]}\left\| \mC \rvx_i^{\text{pub}}\right\|^2 
         + \widehat{G}^{2} \cdot \update{R^2}
         + \overline{G}^2 }_{:= M^2} + d\sigma^{2}
          \Big)\\
        &= \Phi_{t-1}(\theta_{t})
        - 2\eta_t \E_{\rvb_1,\dots,\rvb_{t-1}}[\langle \nabla \gL(\mW \theta_t; D), \mW \theta_t - \theta^{*}\rangle]
        + \eta_t^2 (M^2 + d\sigma^2)\\
    \nonumber
        &= (*)
    \end{align}

    By $\mu$-strong convexity (Assumption~\ref{ass:loss_strongly_convex}), 
    \begin{align}
        \langle \nabla \gL(\mW \theta_t; D), \mW \theta_t - \theta^{*}\rangle
        \geq \gL(\mW \theta_t; D) - \gL(\theta^{*}; D)
        + \frac{\mu}{2}\|\mW \theta_t - \theta^{*}\|^2
    \end{align}
    and
    \begin{align}
        \gL(\mW\theta_t; D) -\gL(\theta^{*}; D) \geq \frac{\mu}{2}\|\mW\theta_t - \theta^{*}\|^2
    \end{align}
    Let $\lambda_{min}(\mA)$ and $\lambda_{max}(\mA)$ denote the minimum and maximum eigenvalue of a positive definite matrix $\mA$, respectively. The above implies
    \begin{align}
        (*)
        &\leq \Phi_{t-1}(\theta_t) - 2 \eta_t \E_{\rvb_1,\dots,\rvb_{t-1}}\left[ \gL(\mW \theta_t; D) - \gL(\theta^{*}; D)
        + \frac{\mu}{2}\|\mW\theta_t - \theta^{*}\|^2\right]
        + \eta_t^2 (M^2 + d\sigma^2) \\
        & \leq \Phi_{t-1}(\theta_t) - 2\eta_t \E_{\rvb_1,\dots,\rvb_{t-1}}\left[ \frac{\mu}{2}\|\mW \theta_t - \theta^{*}\|^2 + \frac{\mu}{2}\|\mW\theta_t - \theta^{*}\|^2 \right] + \eta_t^2 (M^2 + d\sigma^2)\\
        &=  \E_{\rvb_1,\dots,\rvb_{t-1}}[\| \mW\theta_t - \theta^{*} \|_{(\mW\mW^T)^{-1}}^2]
        - 2\mu \eta_t \E_{\rvb_1,\dots,\rvb_{t-1}}[\|\mW \theta_t - \theta^{*}\|^2]
        + \eta_t^2(M^2 + d\sigma^2)
        \\
        &\leq (1- \frac{2 \eta_t \mu}{\lambda_{max}((\mW\mW^{T})^{-1})})
        \E_{\rvb_1,\dots,\rvb_{t-1}}\left[\|\mW\theta_t - \theta^{*}\|_{(\mW\mW^{T})^{-1}}^2\right]
        + \eta_t^2 (M^2 + d\sigma^2)\\
        &= (1- 2 \eta_t \mu \cdot \lambda_{min}((\mW\mW^{T}))
        \E_{\rvb_1,\dots,\rvb_{t-1}}\left[\|\mW\theta_t - \theta^{*}\|_{(\mW\mW^{T})^{-1}}^2\right]
        + \eta_t^2 (M^2 + d\sigma^2)
    \end{align}
    
    Setting $\eta_t = \frac{1}{\mu t \cdot \lambda_{min}(\mW \mW^T)}$, we get
    \begin{equation}
        \Phi_t(\theta_{t+1})
        \leq (1-\frac{2}{t})\Phi_{t-1}(\theta_{t}) + \frac{1}{\mu^2 \lambda^2_{min}(\mW\mW^T) t^2} (M^2 + d\sigma^2)\label{eq:induction_step}
    \end{equation}

    Note that by Lemma~\ref{lemma:strongly_convex_base_case},
    \begin{align}
        \Phi_{0}(\theta_1) 
        &= \|\mW \theta_1 - \theta^{*}\|^2_{(\mW \mW^{T})^{-1}}
        \leq \frac{4}{\mu^2}\cdot\frac{\lambda_{max}(\mW \mW^{T})}{\lambda_{min}(\mW \mW^{T})} (M^2 +  d\sigma^2)\\
    \label{eq:base_case}
        &\leq \frac{4}{\mu^2}(M^2 + d \sigma^{2}) \cdot \underbrace{ \max\{\frac{\lambda_{max}(\mW \mW^{T})}{\lambda_{min}(\mW \mW^T)}, \frac{1}{\lambda_{min}^2(\mW \mW^{T})}\} }_{:=\psi}
    \end{align}

    For simplicity, let $\psi := \max\{\frac{\lambda_{max}(\mW \mW^{T})}{\lambda_{min}(\mW \mW^T)}, \frac{1}{\lambda_{min}^2(\mW \mW^{T})} \}$.
    Let $\sigma_{max}(\mC)$ and $\sigma_{min}(\mC)$ denote the maximum and minimum singular value of the
    matrix $\mC$.
    Note that
    \begin{align}
        &\lambda_{max}(\mW \mW^{T}) = \sigma_{max}^2(\mC), \quad \lambda_{min}(\mW \mW^{T}) = \sigma_{min}^2(\mC) \\
        &\psi = \max\{\frac{\sigma_{max}^2(\mC)}{\sigma^2_{min}(\mC)}, \frac{1}{\sigma^4_{min}(\mC)}\} = \sigma^{-2}_{min}(\mC) \cdot \max\{\sigma_{max}^2(\mC), \sigma^{-2}_{min}(\mC)\}
    \end{align}

    Using Eq.~\ref{eq:base_case} as the base case, and by induction\footnote{
    The induction here follows a similar structure as the one in the proof of Lemma 1 (Appendix B.2) in~\cite{Rakhlin2012strongly_convex_sco}. The main difference lies in the constant.
    Indeed, letting $M' = \frac{4 (M^2 + d\sigma^2)\psi}{\mu^2}$, we have
\begin{align*}
\Phi_t(\theta_{t+1})
&\leq (1-\frac{2}{t})\Phi_{t-1}(\theta_{t}) + \frac{M'}{4t^2} & \text{by Eq.~\ref{eq:induction_step}}\\
&\leq (1-\frac{2}{t})\frac{M'}{t} + \frac{M'}{4t^2} & \text{by induction hypothesis}\\
\end{align*}
To complete the induction, it suffices to prove that $(1 - \frac{2}{t})\frac{1}{t} + \frac{1}{4t^2} \leq \frac{1}{t+1}$, which follows from simple algebra.}
, we can show that
    \begin{align}
    \label{eq:distance}
        \Phi_{t-1}(\theta_{t})
        \leq \frac{4 (M^2 + d\sigma^2)\psi}{\mu^2 t}
    \end{align}

Finally, to analyze average iterate convergence,
    let $\overline{\theta}_t = (\theta_1+\dots,+\theta_t)/t$,
    \begin{align}
        &\E_{\rvb_1,\dots,\rvb_t}\left[\| \mW \overline{\theta}_{t+1} - \theta^{*}\|^2_{(\mW\mW^T)^{-1}}\right]\\
        \nonumber
        &= \E_{\rvb_1,\dots,\rvb_t}\left[\| \frac{t}{t+1}\mW \bar{\theta}_{t} 
        + \frac{1}{t+1} \mW \theta_{t+1}
        - \theta^{*}\|^2_{(\mW\mW^T)^{-1}}
        \right]\\
        &= \E_{\rvb_1,\dots,\rvb_t}\left[\| \frac{t}{t+1}(\mW\overline{\theta}_t - \theta^{*})
        + \frac{1}{t+1}(\mW\theta_{t+1} - \theta^{*})\|^2\right]
        \\
        &= (\frac{t}{t+1})^2\E_{\rvb_1,\dots,\rvb_{t-1}}\left[ \| \mW \overline{\theta}_t - \theta^{*} \|_{(\mW \mW^{T})^{-1}}^2  \right]\\
    \nonumber
        &\quad + \frac{2t}{(t+1)^2}\E_{\rvb_1,\dots,\rvb_t}\left[\langle \mW\overline{\theta}_t - \theta^{*}, \mW\theta_{t+1} - \theta^{*} \rangle_{(\mW\mW^{T})^{-1}} \right]\\
        \nonumber
        &\quad + \frac{1}{(t+1)^2}\E_{\rvb_1,\dots,\rvb_t}\left[\| \mW\theta_{t+1} - \theta^{*} \|_{(\mW \mW^{T})^{-1}}^2\right]\\
        &\stackrel{\text{(a)}}{\leq} (\frac{t}{t+1})^2\E_{\rvb_1,\dots,\rvb_{t-1}}\left[ \| \mW \overline{\theta}_t - \theta^{*} \|_{(\mW \mW^{T})^{-1}}^2  \right]\\
        \nonumber
        &\quad + \frac{2}{t+1} \sqrt{\E_{\rvb_1,\dots,\rvb_{t-1}}\left[\|\mW\overline{\theta}_t - \theta^{*} \|^2_{(\mW\mW^{T})^{-1}}\right]}
        \sqrt{\E_{\rvb_1,\dots,\rvb_t}\left[\|\mW\theta_{t+1} - \theta^{*}\|^2_{(\mW\mW^{T})^{-1}}\right]}\\
        \nonumber
        &\quad + \frac{1}{(t+1)^2}\E_{\rvb_1,\dots,\rvb_{t}}\left[\| \mW\theta_{t+1} - \theta^{*} \|_{(\mW \mW^T)^{-1}}^2\right]\\
    \label{eq:induction_2}
    &\stackrel{\text{(b)}}{\leq}
        (\frac{t}{t+1})^2\E_{\rvb_1,\dots,\rvb_{t-1}}\left[ \| \mW \overline{\theta}_t - \theta^{*} \|_{(\mW \mW^{T})^{-1}}^2  \right]\\
        \nonumber
        &\quad + \frac{4 \sqrt{(M^2 +d\sigma^2)\psi}}{\mu (t+1)^{3/2}}
        \sqrt
        {\E_{\rvb_1,\dots,\rvb_{t-1}}\left[\|\mW \overline{\theta}_t - \theta^{*}\|^2_{(\mW \mW^{T})^{-1}} \right]}
        + \frac{4(M^2+d\sigma^2) \psi}{\mu^2 (t+1)^3}
    \end{align}
    where $(a)$ the last inequality uses the fact that $\E[XY] \leq \sqrt{\E[X^2]}\sqrt{\E[Y^2]]}$ and $(b)$ is by Eq.~\ref{eq:distance}.

    Using Eq.~\ref{eq:base_case} again as the base case,
    and by induction
    \footnote{
    The induction here follows a similar structure as the one in the proof of Theorem 2 (Appendix B.3) in~\cite{Rakhlin2012strongly_convex_sco}. The main difference again lies in the constant. Let $\overline{\Phi}_t(\theta_{t+1}) = \E_{\rvb_1,\dots,\rvb_{t}}[\|\mW \overline{\theta}_{t+1} - \theta^{*} \|^2]$.
    \begin{align*}
        \overline{\Phi}_t(\overline{\theta}_{t+1}) 
        &\leq (\frac{t}{t+1})^2 \Phi_{t-1}(\overline{\theta}_t) + \frac{2\sqrt{M'}}{(t+1)^{3/2}} \sqrt{\Phi_{t-1}(\overline{\theta}_t)} + \frac{M'}{(t+1)^{3}}
        &\text{by Eq.~\ref{eq:induction_2}}\\
        &= \frac{8t M'}{(t+1)^2}
        + \frac{4\sqrt{2} M'}{(t+1)^{3/2}t^{1/2}}
        + \frac{M'}{(t+1)^{3}}
        &\text{by induction hypothesis}
    \end{align*}
    With some algebraic manipulation, it can be shown that $\frac{t}{t+1} + \frac{\sqrt{2}/2}{\sqrt{t(t+1)} } + \frac{1}{8(t+1)^2}  \leq 1$ for $t\geq 2$. This implies that $\frac{8t}{(t+1)^2} + \frac{4\sqrt{2}}{(t+1)^{3/2}\sqrt{t}} + \frac{1}{(t+1)^3}\leq \frac{8}{t+1}$, which concludes the induction.
    }
    ,
    \begin{align}
        \E_{\rvb_1,\dots,\rvb_{T-1}}\left[\|\mW\overline{\theta}_T - \theta^{*}\|^2_{(\mW\mW^{T})^{-1}}\right]
        &\leq \frac{32 (M^2 + d\sigma^2) \psi }{\mu^2 T}
    \end{align}
    Recall that $\mW \theta^{\text{out}} = \mW \overline{\theta}_T $. By smoothness (Assumption~\ref{ass:loss_strongly_convex}),
    \begin{align}
        &\E\left[\gL(\mW \theta^{\text{out}};D)\right]
        -\gL^{*}
        \leq \frac{\beta}{2} \E\left[\| \mW \theta^{\text{out}} - \theta^{*} \|^2\right]\\
        &\leq \frac{\beta}{2}\frac{1}{\lambda_{min}((\mW \mW^{T})^{-1})}\E\left[\|\mW \theta^{\text{out}} - \theta^{*}\|_{(\mW \mW^{T})^{-1
        }}^2\right]\\
        &\leq \frac{\beta}{2} \psi \lambda_{max}(\mW \mW^{T}) \frac{32 (M^2 + d \sigma^2)}{\mu^2 T} 
    \label{eq:convergence_strongly_convex_raw}
        = \frac{\beta}{2} \psi \sigma^2_{max}(\mC) \frac{32 (M^2 + d \sigma^2)}{\mu^2 T} 
    \end{align}

    Setting the noise variance as $\sigma^2 = \widetilde{\gO}(\frac{M^2 T}{n^2 \eps^2})$, and following Eq.~\ref{eq:convergence_strongly_convex_raw},
    \begin{align}
        &\E\left[\gL(\mW \theta^{\text{out}};D)\right]
        -\gL^{*}
        \leq (\frac{\sigma_{max}(\mC)}{\sigma_{min}(\mC)})^2 \max\{\sigma_{max}^2(\mC), \sigma^{-2}_{min}(\mC)\} \frac{16 \beta}{\mu^2} M^2 (\frac{1}{T} + \frac{d}{n^2 \eps^2})
    \end{align}

    Converting the above bound back to the matrix form of parameters to obtain the final theorem statement.
    
\end{proof}

\section{Convergence of Private \update{Models with Linear Input Transformations} Under Non-Convex Losses}
\label{sec:appendix_convergence_non_convex}

\begin{theorem}[Re-statement of Theorem~\ref{thm:convergence_non_convex}]
    Let $\mC$ be a conditioning matrix of full rank.
    Under Assumption~\ref{ass:lipschitzness} and~\ref{ass:smoothness_semi_norm}, if Algorithm~\ref{alg:cond_dp} is run with $\eta_t = \eta = \frac{1}{\sqrt{T} \beta}, \forall t$, and noise variance $\sigma^2 = \widetilde{O}(\frac{M^2 T}{\eps^2 n^2 })$, then it guarantees
    \begin{align}
        &\mathbb{E}[\|\nabla \gL(\Theta^{\text{pub-out}}\mC, \Theta^{\text{priv-out}}, \omega^{\text{out}} ; D)\|^2]\\
        \nonumber
        &\leq \frac{\beta_0}{\sqrt{T}} \Big(\gL(\Theta_0^{\text{pub}}\mC, \Theta_0^{\text{priv}}, \omega_0; D) - \gL^{*} \Big)
        + \Big( G^2 \max_{i\in [n]}\| \mathbf{C}\mathbf{x}_i^{\text{pub}}\|^2 + \widehat{G}^2 \update{R^2} + \overline{G}^2\Big)\cdot \Big(\frac{\beta_0}{2\sqrt{T}} + \frac{d}{2n^2\epsilon^2 \sqrt{T}} \Big)
    \end{align}
    
\end{theorem}

\begin{proof}
Following the notations we defined in Appendix~\ref{sec:appendix_convergence_glm},
recall that the loss function is $L(\Theta^{\text{pub-out}}, \Theta^{\text{priv-out}}, \omega; \gD)$, where $\Theta^{\text{pub-out}}, \Theta^{\text{priv-out}}, \omega$ denote model parameter corresponding to the public features, the private features in the bottom embedding layer, and model parameters in the upper prediction layer, respectively. $D$ is the training dataset. Also recall that the loss $\gL$ can be decomposed into $\gL_i$'s, where each $\gL_i$ depends on on data sample $\mathbf{d}_i \in D$. See Section 3 for details.
$\theta_t$ and $\theta^{\text{out}}$ is the vector version of all parameters at iteration $t$ and in the output, respectively.

Consider a constant learning rate $\eta$, by the smoothness assumption in Assumption~\ref{ass:smoothness_semi_norm}, we have
\begin{align}
    \gL(\mW \theta_{t+1}; D)
    &\leq \gL(\mW \theta_{t}; D) + \langle \nabla \gL(\mW \theta_t; D), \mW \theta_{t+1} - \mW \theta_{t} \rangle_{(\mW \mW^{T})^{-1}} + \frac{\beta_0}{2} \| \mW\theta_{t+1} - \mW\theta_{t}\|_{(\mW\mW^{T})^{-1}}^2\\
    &= \gL(\mW \theta_{t}; D)
    -\eta \langle \nabla \gL(\mW \theta_t; D), \mW (g_t + b_t) \rangle_{(\mW \mW^{T})^{-1}} + \frac{\beta_0}{2} \eta^2 \| \mW(g_t + b_t )\|_{(\mW \mW^{T})^{-1}}^2\\
    &= \gL(\mW \theta_{t}; D)
    - \eta (\nabla \gL(\mW\theta_t;D))^T (\mW \mW^{T})^{-1} \mW(\rvg_t + \rvb_t)
    + \frac{\beta_0}{2}\eta^{2} \|\mW(\rvg_t + \rvb_t)\|^2_{(\mW \mW^{T})^{-1}}
\end{align}
Taking expectation, the above is
\begin{align}
     \E[\gL(\mW \theta_{t+1}; D)]
    &\leq \E[\gL(\mW \theta_{t}; D)]
    -\eta (\nabla \gL(\mW\theta_t;D))^T (\mW \mW^{T})^{-1} \mW \rvg_t\\
    \nonumber
    &\quad + \frac{\beta_0}{2}\eta^{2} \mathbb{E}[\| \mW^{T} \rvg_t \|_{(\mW \mW^{T})^{-1}}^2] + \frac{\beta_0}{2}\eta^2 \mathbb{E}[\| \mW\mathbf{b}_t\|_{(\mW \mW^{T})^{-1}}^2] \\
    &= \mathbb{E}[\gL(\mW \theta_{t}; D)]
    -\eta \mathbb{E}[ (\nabla \gL(\mW\theta_t;D))^T (\mW \mW^{T})^{-1} \mW \mW^{T}\nabla \gL(\mW\theta_t; D) ]\\
    \nonumber
    &\quad + \frac{\beta_0}{2} \eta^2 \E[\| \mW^T \mW \nabla \gL(\mW\theta_t; D)\|^2_{(\mW \mW^T)^{-1}}]
    + \frac{\beta_0}{2}\eta^2 \E[\| \mW \mathbf{b}_t\|^2_{(\mW \mW^T)^{-1}}]
\end{align}

By Eq.(\ref{eq:non_convex_ref_1_start}) - Eq.(\ref{eq:non_convex_ref_1_end}),
$\E[\| \mW \mathbf{b}_t\|^2_{(\mW \mW^T)^{-1}}] = d\sigma^2$.

By Eq.(\ref{eq:non_convex_ref_2_start}) - Eq.(\ref{eq:grad_norm_bound}),
$\E[\| \mW^T \mW \nabla \gL(\mW\theta_t; D)\|^2_{(\mW \mW^T)^{-1}}] \leq G^2 \max_{i\in [n]}\| \mathbf{C}\mathbf{x}_i^{\text{pub}}\|^2 + \widehat{G}^2 \update{R^2} + \overline{G}^2:= M^2$.

Hence,
\begin{align}
    \E[\gL(\mW\theta_{t+1}; D)]
    &\leq \E[\gL(\mW\theta_t; D)] - \eta \mathbb{E}[\|\nabla \gL(\mW\theta_t; D)\|^2]
    + \frac{\beta_0 \eta^2}{2} M^2 + \frac{\beta_0 \eta^2}{2} d \sigma^2\\
    \mathbb{E}[\|\nabla \gL(\mW\theta_t; D)\|^2]
    &\leq \frac{1}{\eta}\Big( \mathbb{E}[\gL(\mW \theta_{t}; D)] - \E[\gL(\mW \theta_{t+1}; D)]  \Big)
    + \frac{\beta_0 \eta}{2}(M^2 + d\sigma^2)
\end{align}
and so
\begin{align}
    \frac{1}{T} \sum_{t=0}^{T-1} \mathbb{E}[\| \nabla \gL(\mW\theta_t; D)\|^2]
    &\leq \frac{1}{T \eta}\Big(\gL(\mW\theta_0; D) - \mathbb{E}[\gL(\mW \theta_{T};D)] \Big)
    + \frac{\beta_0 \eta}{2}(M^2 + d\sigma^2)\\
    &\leq \frac{1}{T \eta}\Big(\gL(\mW\theta_0; D) - \gL^{*} \Big)
    + \frac{\beta_0 \eta}{2}(M^2 + d\sigma^2)
\end{align}

Setting the noise variance as $\sigma^2 = \widetilde{O}(\frac{M^2 T}{n^2 \eps^2})$, it follows that
\begin{align}
    \frac{1}{T} \sum_{t=0}^{T-1} \mathbb{E}[\| \nabla \gL(\mW\theta_t; D)\|^2]
    &\leq \frac{1}{T \eta}\Big(\gL(\mW\theta_0; D) - \gL(\theta^{*};D) \Big)
    + \frac{\beta_0 \eta}{2}(M^2 + d \frac{M^2 T}{n^2 \eps^2})
\end{align}
Choosing $\eta = O(\frac{1}{\sqrt{T}\beta})$, there exists $\widehat{t} \in \{0, 1,\dots, T-1\}$ such that
\begin{align}
    &\mathbb{E}[\|\nabla \gL(\mW\theta^{\text{out}}; D)\|^2]
    = \mathbb{E}[\|\nabla \gL(\mW\theta_{\widehat{t}}; D)\|]^2]\\
    \nonumber
    &\leq \frac{\beta_0}{\sqrt{T}} \Big(\gL(\mW\theta_0; D) - \gL^{*} \Big)
    + M^2\Big(\frac{\beta_0}{2\sqrt{T}} + \frac{d}{2n^2\epsilon^2 \sqrt{T}} \Big)
\end{align}
Converting the above bound back to the matrix form of parameters yields the final theorem statement.
\end{proof}

\section{Private Linear Regression}
\label{sec:appendix_priv_lin_reg}

\begin{lemma}[Block Matrix Inversion]
\label{lemma:block_matrix_inversion}
    Given a matrix $\mX = \begin{bmatrix}
        \mA & \mB\\
        \mC & \mD
    \end{bmatrix}$, 
    if $\mA^{-1}$ or $\mD^{-1}$ exists, then
    \begin{align}
        \mX^{-1} = \begin{bmatrix}
            \mA^{-1} + \mA^{-1} \mB \mS^{-1} \mC \mA^{-1}
            & - \mA^{-1} \mB \mS^{-1}\\
            - \mS^{-1} \mC \mA^{-1} & \mS^{-1}
        \end{bmatrix}
    \end{align}
    where $\mS = \mD - \mC \mA^{-1}\mB$.
\end{lemma}

\begin{lemma}[Linear Regression Optimum]
\label{lemma:lin_reg_optimum}
    Consider solving:
    \begin{align*}
        \min_{\theta^{\text{pub}}, \theta^{\text{priv}}} \|(\mX^{\text{pub}} \theta^{\text{pub}} + \mX^{\text{priv}}\theta^{\text{priv}}) - \rvy\|^2
    \end{align*}
    where $\mX^{\text{pub}} \in \R^{n \times d^{\text{pub}}}$, $\mX^{\text{priv} \in \R^{n \times d^{\text{priv}}}}$ and $\rvy\in \R^{n}$ represent the public features, private features and the labels, respectively.
    Let $\mX^{\text{pub}} = \mU \Sigma \mV^{T}$ be its SVD decomposition. 
    Then the optimal solution $(\theta^{\text{pub}})^{*}$ and $(\theta^{\text{pub}})^{*}$ is given by
    \begin{align}
        (\theta^{\text{pub}})^* 
        &= \mV \Sigma^{-1}\Big( \sI_n
        + \mU^T \mX^{\text{priv}} \Big[ (\mX^{\text{priv}})^T (\sI_n - \mU \mU^T) \mX^{\text{priv}} \Big]^{-1} (\mX^{\text{priv}})^T \mU \Big)
        \mU^T \rvy\\
         (\theta^{\text{priv}})^{*} 
         &= \Big[ (\mX^{\text{priv}})^T (\sI_n - \mU \mU^T) \mX^{\text{priv}} 
        \Big]^{-1} (\mX^{\text{priv}})^T \Big(
            -\mU \mU^{T} + \sI_n
        \Big) \rvy
    \end{align}
\end{lemma}

\begin{proof}
    For ease of presentation, we again write all parameters as a vector. 
    Let
    \begin{align*}
        \theta^{*} 
        &= \begin{bmatrix}
            (\theta^{\text{pub}})^{*}\\
            (\theta^{\text{priv}})^{*}
        \end{bmatrix}
        = \argmin_{\theta^{\text{pub}}, \theta^{\text{priv}}} \left\| \begin{pmatrix}
            \mX^{\text{pub}} & \mX^{\text{priv}}
        \end{pmatrix} 
        \begin{pmatrix}
            \theta^{\text{pub}}\\
            \theta^{\text{priv}}
        \end{pmatrix} - \rvy
        \right\|^2\\
        &= \argmin_{\theta}\left\| \mX \theta - \rvy\right\|^2
    \end{align*}
    
    It is well known that 
    \begin{align}
        \theta^{*} &= (\mX^{T} \mX)^{-1} \mX^{T} \rvy
        = \left( \begin{bmatrix}
            (\mX^{\text{pub}})^T\\
            (\mX^{\text{priv}})^T
        \end{bmatrix} \begin{bmatrix}
            \mX^{\text{pub}} & \mX^{\text{priv}}
        \end{bmatrix}
        \right)^{-1} 
        \begin{bmatrix}
            (\mX^{\text{pub}})^T\\
            (\mX^{\text{priv}})^T
        \end{bmatrix} \rvy\\
        &= \left( \begin{bmatrix}
            (\mX^{\text{pub}})^T \mX^{\text{pub}}
            & (\mX^{\text{pub}})^T \mX^{\text{priv}}\\
            (\mX^{\text{priv}})^T \mX^{\text{pub}}
            & (\mX^{\text{priv}})^T \mX^{\text{priv}}
        \end{bmatrix}\right)^{-1}\begin{bmatrix}
            (\mX^{\text{pub}})^T\\
            (\mX^{\text{priv}})^T
        \end{bmatrix} \rvy
    \end{align}
    By Lemma~\ref{lemma:block_matrix_inversion}, let 
    \begin{align}
    (\mX^T \mX)^{-1} 
    &:= \begin{bmatrix}
        \mA & \mB\\
        \mC & \mD
    \end{bmatrix}\\
    \end{align}
    where
    \begin{align}
    \mA &= ((\mX^{\text{pub}})^T \mX^{\text{pub}})^{-1} + ((\mX^{\text{pub}})^T \mX^{\text{pub}})^{-1} (\mX^{\text{pub}})^T \mX^{\text{priv}} \mS^{-1} (\mX^{\text{priv}})^T \mX^{\text{pub}} ((\mX^{\text{pub}})^T \mX^{\text{pub}})^{-1}\\
    \mB &=  - ((\mX^{\text{pub}})^T \mX^{\text{pub}})^{-1} (\mX^{\text{pub}})^T \mX^{\text{priv}} \mS^{-1}\\
    \mC &= - \mS^{-1} (\mX^{\text{priv}})^T\mX^{\text{pub}} ((\mX^{\text{pub}})^T \mX^{\text{pub}})^{-1}\\
    \mD &= \mS^{-1}\\
    \mS &= (\mX^{\text{priv}})^T \mX^{\text{priv}} - (\mX^{\text{priv}})^T \mX^{\text{pub}} ((\mX^{\text{pub}})^T \mX^{\text{pub}})^{-1} (\mX^{\text{pub}})^T \mX^{\text{priv}}
\end{align}

Recall that $\mX^{\text{pub}} = \mU \Sigma \mV^{T}$ is its SVD decomposition, where $\mU \in \R^{n \times d^{\text{pub}}}$, $\Sigma \in \R^{d^{\text{pub}} \times d^{\text{pub}}}$ and $\mV \in \R^{d^{\text{pub}} \times d^{\text{pub}}}$. Following the above,

\begin{align}
    &((\mX^{\text{pub}})^T \mX^{\text{pub}})^{-1} = (\mV \Sigma \mU^T \mU \Sigma \mV^{T})^{-1} = \mV \Sigma^{-2} \mV^T\\
    & ((\mX^{\text{pub}})^T \mX^{\text{pub}})^{-1} (\mX^{\text{pub}})^T = \mV \Sigma^{-2} \mV^T \mV \Sigma \mU^T = \mV \Sigma^{-1} \mU^T\\
    & \mX^{\text{pub}} ((\mX^{\text{pub}})^T \mX^{\text{pub}})^{-1} = \mU \Sigma \mV^T \mV \Sigma^{-2} \mV^T = \mU \Sigma^{-1} \mV^T\\
    & \mX^{\text{pub}} ((\mX^{\text{pub}})^T \mX^{\text{pub}})^{-1} (\mX^{\text{pub}})^T = \mU \Sigma^{-1} \mV^T \mV \Sigma \mU^T = \mU \mU^T \\
    &\mS = (\mX^{\text{priv}})^T \Big( \sI_{n} - \mX^{\text{pub}} ((\mX^{\text{pub}})^T \mX^{\text{pub}})^{-1} (\mX^{\text{pub}})^T \Big) \mX^{\text{priv}}
    = (\mX^{\text{priv}})^T \Big( \sI_{n} - \mU \mU^T \Big) \mX^{\text{priv}}
\end{align}
And
\begin{align}
    \theta^* = \begin{bmatrix}
        \mA  & \mB\\
        \mC & \mD
    \end{bmatrix} \begin{bmatrix}
        (\mX^{\text{pub}})^T\\
        (\mX^{\text{priv}})^T
    \end{bmatrix} \rvy
    = \begin{bmatrix}
        (\mA (\mX^{\text{pub}})^T + \mB (\mX^{\text{priv}})^T) \rvy\\
        (\mC (\mX^{\text{pub}})^T + \mD (\mX^{\text{priv}})^T) \rvy
    \end{bmatrix}
    = \begin{bmatrix}
        (\theta^{\text{pub}})^*\\
        (\theta^{\text{priv}})^*
    \end{bmatrix}
\end{align}

Note that
\begin{align}
    \mA (\mX^{\text{pub}})^T 
    &= ((\mX^{\text{pub}})^T \mX^{\text{pub}})^{-1} (\mX^{\text{pub}})^T \\
    \nonumber
    &\quad + ((\mX^{\text{pub}})^T \mX^{\text{pub}})^{-1} (\mX^{\text{pub}})^T \mX^{\text{priv}} \mS^{-1} (\mX^{\text{priv}})^T \mX^{\text{pub}} ((\mX^{\text{pub}})^T \mX^{\text{pub}})^{-1} (\mX^{\text{pub}})^T\\
    &= \mV \Sigma^{-1} \mU^T
    + \mV \Sigma^{-1} \mU^T \mX^{\text{priv}} \Big[ (\mX^{\text{priv}})^T (\sI_n - \mU \mU^T) \mX^{\text{priv}} \Big]^{-1} (\mX^{\text{priv}})^T \mU \mU^T\\
    \mB (\mX^{\text{priv}})^T
    &= - ((\mX^{\text{pub}})^T \mX^{\text{pub}})^{-1} (\mX^{\text{pub}})^T \mX^{\text{priv}} \mS^{-1} (\mX^{\text{priv}})^T\\
    &= - \mV \Sigma^{-1} \mU^T \mX^{\text{priv}} \Big[ (\mX^{\text{priv}})^T (\sI_n  \mU \mU^T) \mX^{\text{priv}} \Big]^{-1} (\mX^{\text{priv}})^T
\end{align}

Therefore,
\begin{align}
    (\theta^{\text{pub}})^* = \mV \Sigma^{-1}\Big( \sI_n
    + \mU^T \mX^{\text{priv}} \Big[ (\mX^{\text{priv}})^T (\sI_n - \mU \mU^T) \mX^{\text{priv}} \Big]^{-1} (\mX^{\text{priv}})^T \mU \Big)
    \mU^T \rvy
\end{align}

Similarly,
\begin{align}
    \mC (\mX^{\text{pub}})^T &= -\Big[ (\mX^{\text{priv}})^T (\sI_n - \mU \mU^T) \mX^{\text{priv}} \Big]^{-1} (\mX^{\text{priv}})^T \mU \mU^T \\
    \mD (\mX^{\text{priv}})^T &= \Big[ (\mX^{\text{priv}})^T (\sI_n - \mU \mU^T) \mX^{\text{priv}} \Big]^{-1} (\mX^{\text{priv}})^T
\end{align}

and

\begin{align}
    (\theta^{\text{priv}})^{*} = \Big[ (\mX^{\text{priv}})^T (\sI_n - \mU \mU^T) \mX^{\text{priv}} 
    \Big]^{-1} (\mX^{\text{priv}})^T \Big(
        -\mU \mU^{T} + \sI_n
    \Big) \rvy
\end{align}

\end{proof}

\begin{lemma}[Re-statement of Lemma~\ref{lemma:improvement_C_lin_reg}]
   Consider the linear regression problem defined in (\ref{eq:lin_reg_problem}), and let the singular value decomposition (SVD) of the public feature matrix be $\mX^{\text{pub}} = \mU \Sigma \mV^{T}$. Suppose $\precond$ (Algorithm~\ref{alg:cond_dp}) is run with $\eta_t, \sigma^2$ as in Theorem~\ref{thm:appendix_convergence_glm}, then its excess risk is bounded as follows. If $\mC = \widehat{\mC} = \mV \Sigma^{-1}\mV^{T}$, then
    \begin{align}
        &\E\left[\gL(\theta^{\text{pub-out}}, \theta^{\text{priv-out}}; D)\right]
        - \gL^*
        \leq \sqrt{\frac{1}{T} + \frac{d}{n^2 \eps^2}} \cdot \sqrt{ 2 \| \widehat{\rvy} \|^2 + F^2}
        \cdot
        \sqrt{
            G^2 + E^2
        }
    \end{align}
    and if $\mC = \sI_{d^{\text{pub}}}$ (corresponding to the \dpsgd\ baseline), then
    \begin{align}
        &\E\left[\gL(\theta^{\text{pub-out}}, \theta^{\text{priv-out}}; D)\right]
        - \gL^*
        \leq \sqrt{\frac{1}{T} + \frac{d}{n^2 \eps^2}}
        \cdot \sqrt{2 \|\Sigma^{-1} \widehat{\rvy}\|^2 + F^2
        }
        \cdot \sqrt{G^2 \sigma_{\text{max}}^2 + E^2}
    \end{align}    
    where $\sigma_{\text{max}}$ is the largest singular value of $\mX^{\text{pub}}$, and
    \begin{align}
        F^2 &=  2\|\theta_0^{\text{pub}}\|^2 
        + \|\theta_0^{\text{priv}} - (\theta^{\text{priv}})^{*}\|^2\\
        E^2 &= \widehat{G}^{2} \cdot \update{R^2}\\
        \widehat{\rvy} &= ( \sI_n
        + \mU^T \mX^{\text{priv}} \Big[ (\mX^{\text{priv}})^T (\sI_n - \mU \mU^T) \mX^{\text{priv}} \Big]^{-1} (\mX^{\text{priv}})^T \mU )
        \mU^T \rvy
    \end{align}
\end{lemma}

\begin{proof}
    Let $\mX^{\text{pub}} = \mU \Sigma \mV^T$ be its SGD decomposition.
    Consider the following preconditioning matrix $\widehat{\mC}$ as 
    \begin{align}
        \widehat{\mC} = 
            \mV \Sigma^{-1} \mV^T,
        \text{ and }
        (\mC)^{-1} = (\mC^{T})^{-1} = \mV \Sigma \mV^{T}
    \end{align}

    Since $\omega = \emptyset$, $\overline{G} = 0$ in this case.
    By Theorem~\ref{thm:convergence_glm}, the excess risk is bounded as follows,
    \begin{align}
        &\E\left[\gL(\theta^{\text{pub-out}}, \theta^{\text{priv-out}}; D)\right]
        - \gL^*
        \leq \sqrt{\frac{1}{T} + \frac{d}{n^2\eps^2} } \cdot \sqrt{G^2 \cdot \max_{i\in [n]}\|\mC \rvx_i^{\text{pub}}\|^2 
        + \widehat{G}^{2} \cdot \update{R^2}}\\
        \nonumber
        &\cdot \sqrt{
        \|\mC^{T} \theta_0^{\text{pub}} - (\theta^{\text{pub}})^{*}\|^2_{(\mC^T \mC)^{-1}}
        + \|\theta^{\text{priv}}_0 - (\theta^{\text{priv}})^*\|^2
        }\\
    \label{eq:convergence_bound_summary}
        &\leq \sqrt{\frac{1}{T} + \frac{d}{n^2\eps^2} } \cdot \sqrt{G^2 \cdot \max_{i\in [n]}\|\mC \rvx_i^{\text{pub}}\|^2 
        + \widehat{G}^{2} \cdot \update{R^2}}\\
        \nonumber
        &\quad \cdot \sqrt{2\| \mC^{T}\theta_0^{\text{pub}}\|_{(\mC^T \mC)^{-1}}^2
        + 2\|(\theta^{\text{pub}})^{*}\|_{(\mC^T \mC)^{-1}}^2
        + \|\theta^{\text{priv}}_0 - (\theta^{\text{priv}})^*\|^2
        }
    \end{align}

    By Lemma~\ref{lemma:lin_reg_optimum}, the vector version of $(\Theta^{\text{pub}})^*$ is given by
    \begin{align}
        (\theta^{\text{pub}})^* 
        &= \mV \Sigma^{-1} \underbrace{ \Big( \sI_n
        + \mU^T \mX^{\text{priv}} \Big[ (\mX^{\text{priv}})^T (\sI_n - \mU \mU^T) \mX^{\text{priv}} \Big]^{-1} (\mX^{\text{priv}})^T \mU \Big)}_{:= \mA}
        \mU^T \rvy\\
        &= \mV \Sigma^{-1} \mA \mU^T \rvy
    \end{align}

    Let $\sigma_{min}$ and $\sigma_{max}$ be the minimum and maximum absolute singular values of $\mX^{\text{pub}}$.

    \textbf{The \dpsgd\ Baseline.}
    Here, $\mC = \sI_{d^{\text{pub}}}$.
    \begin{align}
        &\|\mC^{T} \theta_0^{\text{pub}}\|_{(\mC^{T}\mC)^{-1}}^{2} = \|\theta_0^{\text{pub}}\|^2\\
        & \|(\theta^{\text{pub}})^{*}\|_{(\mC^T \mC)^{-1}}^2
        = \|(\theta^{\text{pub}})^{*}\|^2
        = \|\mV \Sigma^{-1} \mA \mU^{T}\rvy\|^2
        = \|\Sigma^{-1}\mA \mU^{T}\rvy\|^2
        \\
        &\max_{i\in [n]} \|\mC \rvx_i^{\text{pub}}\|^2
        = \max_{i\in [n]}\| \mV\Sigma \rvu_i\|^2
        = \max_{i\in [n]} \|\Sigma \rvu_i\|^2
        \leq \sigma_{max}^2
    \end{align}
    Following Eq.~\ref{eq:convergence_bound_summary}, the excess risk of \dpsgd\ is bounded by

    \begin{align}
         &\E\left[\gL(\theta^{\text{pub-out}}, \theta^{\text{priv-out}}; D)\right]
        - \gL^*\\
        \nonumber
        &\leq \sqrt{\frac{1}{T} + \frac{d}{n^2 \eps^2}} \cdot \sqrt{G^2 \sigma_{max}^2 
         + \widehat{G}^{2} \cdot \update{R^2}} \cdot 
        \sqrt{2 \|\theta_0^{\text{pub}}\|^2
        + 2 \|\Sigma^{-1}\mA \mU^T \rvy\|^2 + \|\theta^{\text{priv}}_0 - (\theta^{\text{priv}})^*\|^2}
    \end{align}

    \textbf{$\precond$.} Under the choice $\mC = \widehat{\mC} = \mV \Sigma^{-1}\mV^T$, we have
    \begin{align}
        &\|\mC^{T}\theta_0^{\text{pub}}\|_{(\mC^T \mC)^{-1}}^2
        = (\theta_0^{\text{pub}})^{T} \mC (\mC^T \mC)^{-1} \mC^{T} \theta_0^{\text{pub}}\\
        &\quad = (\theta_0^{\text{pub}})^{T} \mV \Sigma^{-1} \mV^T \mV \Sigma^{2}\mV^T \mV \Sigma^{-1} \mV^T \theta_0^{\text{pub}}
        = \|\theta_0^{\text{pub}}\|^2 \\
        &\|(\theta^{\text{pub}})^{*}\|_{(\mC^T \mC)^{-1}}^2
        = ((\theta^{\text{pub}})^{*})^T (\mC^{T } \mC)^{-1} (\theta^{\text{pub}})^{*}\\
        &\quad = \rvy \mU \mA \Sigma^{-1} \mV^{T} \mV^{T} \Sigma^{2} \mV^{T} \mV \Sigma^{-1} \mA \mU^{T} \rvy
        = \|\mA \mU^T \rvy\|^2\\
        &\max_{i\in [n]} \|\mC\rvx_i^{\text{pub}}\|^2
        = \max_{i\in [n]} \|\mV \Sigma^{-1}\mV^{T} \mV \Sigma \rvu_i\|^2
        = \max_{i\in [n]} \| \rvu_i \|^2 \leq 1
    \end{align}

    Following Eq.~\ref{eq:convergence_bound_summary}, the excess risk of $\precond$ is
    \begin{align}
    \label{eq:convergence_precond}
        &\E\left[\gL(\theta^{\text{pub-out}}, \theta^{\text{priv-out}}; D)\right]
        - \gL^*\\
        \nonumber
        &\leq \sqrt{\frac{1}{T} + \frac{d}{n^2\eps^2}}
        \cdot \sqrt{G^{2} 
         + \widehat{G}^{2} \cdot \update{R^2}}
        \cdot \sqrt{2\|\theta_0^{\text{pub}}\|^2
        + 2\|\mA \mU^T \rvy\|^2 + \|\theta^{\text{priv}}_0 - (\theta^{\text{priv}})^*\|^2}
    \end{align}

\end{proof}

\section{Comparison with \rronbins}
\label{sec:appendix_comparison_rronbins}

The main state-of-the-art method for regression under label DP is \rronbins~\cite{ghazi2023regression_ldp}.
Although both \precond\ and \rronbins\ aim to improve utility in label-DP settings, they differ in several key aspects.

First, \precond\ provides better privacy–utility trade-offs, particularly in low-privacy regimes, as it provides approximate DP guarantees, whereas \rronbins\ provides the more strict pure DP guarantees, whereas \precond\ satisfies approximate differential privacy.

Second, \precond\ adopts a simpler solution. \rronbins\ requires discretizing labels via binning and privately estimating a prior distribution, both of which introduce additional design choices and hyperparameters, including the allocation of the privacy budget for prior estimation, making it hard to deploy in practice. In contrast, \precond\ relies on a single conditioning step constructed from public features and integrates directly with standard DPSGD-style optimization.

Finally, the computational profiles of the two methods differ substantially. The dominant cost of \precond\ arises from computing an SVD of the public feature matrix, with complexity $(d^{\text{pub}})^3$, whereas \rronbins\ incurs a computational cost of $(n^{\text{pub}})^2$. This cubic dependence on the feature dimension can significantly limit the scalability of \rronbins\ in high-dimensional settings.

\section{More About Experiments}
\label{sec:appendix_exp}

\subsection{Baselines}
\label{subsec:appendix_baselines}
\begin{itemize}[leftmargin=6mm, topsep=0mm]
    \item \textbf{\dpsgd}~\cite{Abadi2016dpsgd}: Applying \dpsgd\ to train model using private and public features along with private labels. \dpsgd\ provides approximate $(\epsilon, \delta)$-DP guarantees.

    \item \textbf{\rronbins}~\cite{ghazi2023regression_ldp}: 
    The algorithm first estimates a prior $\Phi$ over label distributions using a portion of the privacy budget, then adds noise to the labels based on this prior to minimize the distance to the true labels.
    The model is trained using public features and privatized labels, achieving pure $\eps$-DP guarantees.
    Following~\cite{ghazi2023regression_ldp}, we allocate $0.25\eps$ for prior estimation and the remaining $0.75\eps$ for label privatization.

    \item \textbf{\wtdllp}~\cite{brahmbhatt2023ldp_aggregation}: 
    This algorithm is only
    applicable to private linear regression under label DP and provides approximate $(\eps, \delta)$-DP guarantees. 
    It first adds Gaussian noise to a random fraction $\rho \in (0, 1]$ of training labels. 
    Then, it selects a fraction $\alpha \in (0, 1)$ of training samples,
    partitions them into bags, and aggregates each bag into a single noisy feature-label pair. 
    The resulting noisy dataset is used to train the model via SGD.
    Following~\cite{brahmbhatt2023ldp_aggregation}, we set $\alpha = 0.6$ and treat $\rho \in \{0.01, 0.5, 1\}$ as a tunable hyperparameter.

\end{itemize}

\subsection{Additional Experiment Details}
\label{subsec:appendix_add_exp_details}

\textbf{Hardware.} Most experiments were conducted on AWS g6e.48xlarge\footnote{\url{https://aws.amazon.com/ec2/instance-types/g6e/}} instances equipped with NVIDIA L40S Tensor Core GPUs.

\subsection{Datasets Spectra}
\label{subsec:appendix_dataset_spectra}

We showed the singular spectrum of the \texttt{synthetic}, \texttt{wine} and \texttt{energy} dataset in Figure~\ref{fig:spectra_synthetic} and Figure~\ref{fig:spectrum_datasets}, respectively. Here, we present the spectra of the remaining datasets:
\texttt{Boston\_housing}, \texttt{CA\_housing}, \texttt{supercond}, \texttt{TomsHardware}, \texttt{wave} and \texttt{Criteo Sponsored Search Conversion} in Figure~\ref{fig:spectra_datasets}.

\begin{figure}[h]
    \centering
    \includegraphics[width=0.28\linewidth]{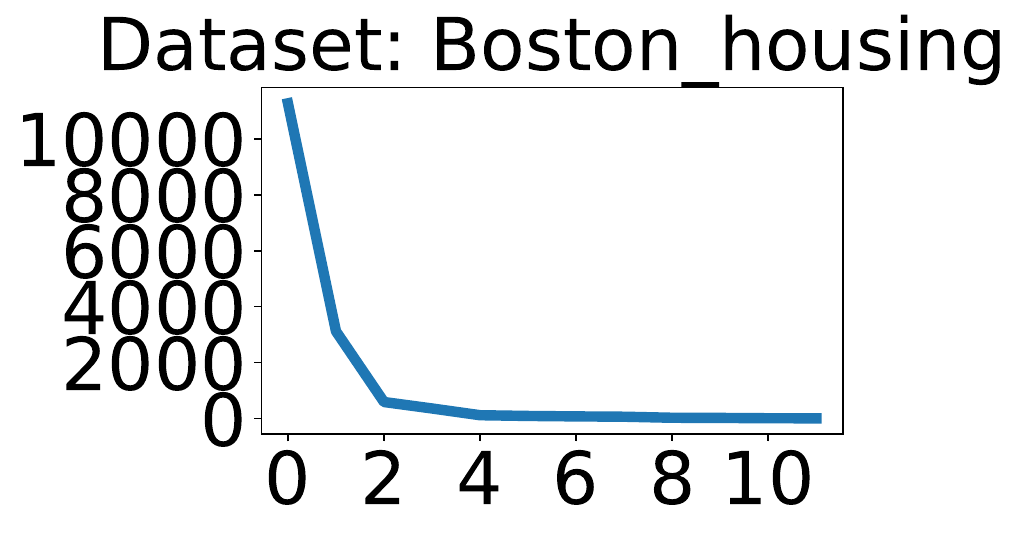}
    \includegraphics[width=0.23\linewidth]{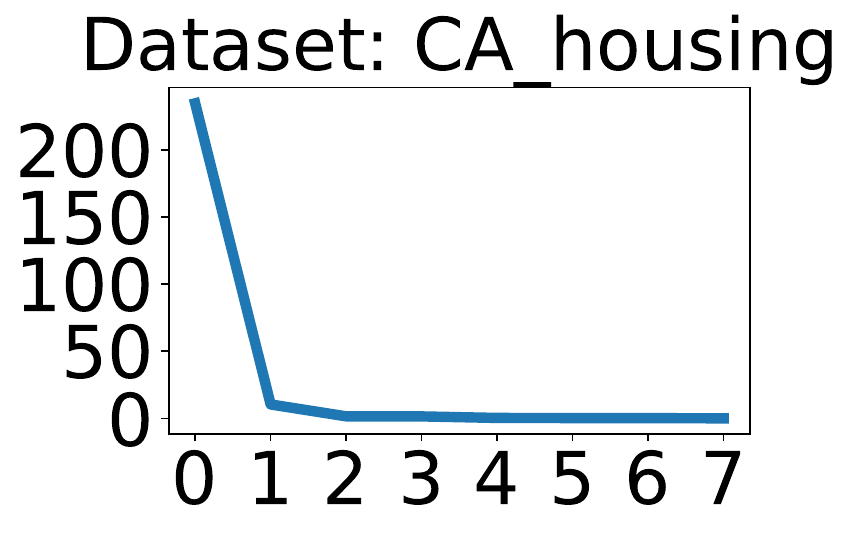}
    \includegraphics[width=0.24\linewidth]{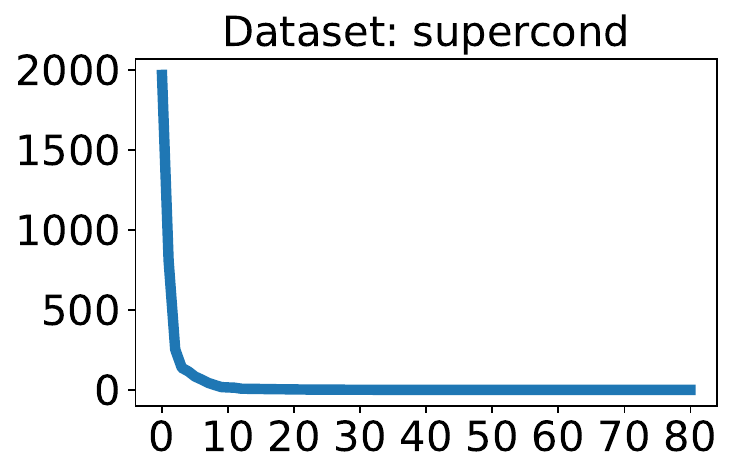}\\
    \includegraphics[width=0.24\linewidth]{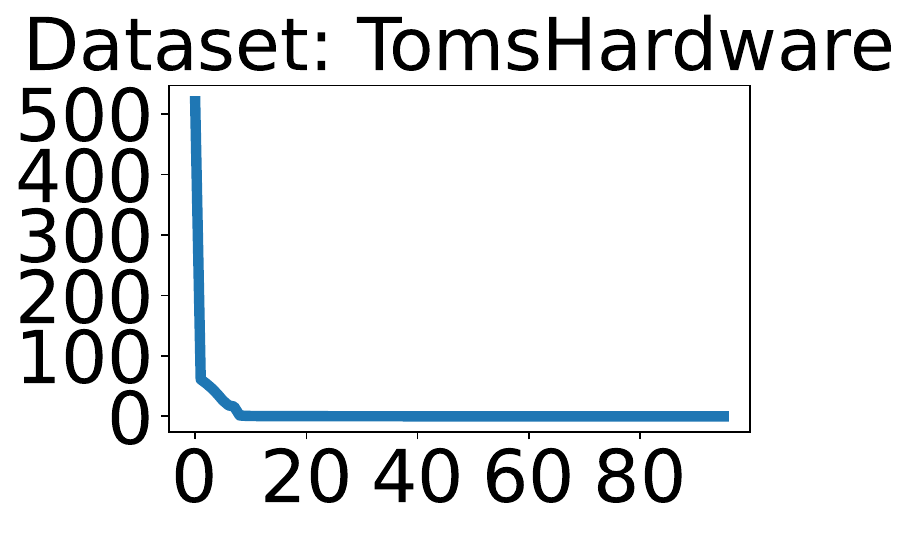}
    \includegraphics[width=0.26\linewidth]
    {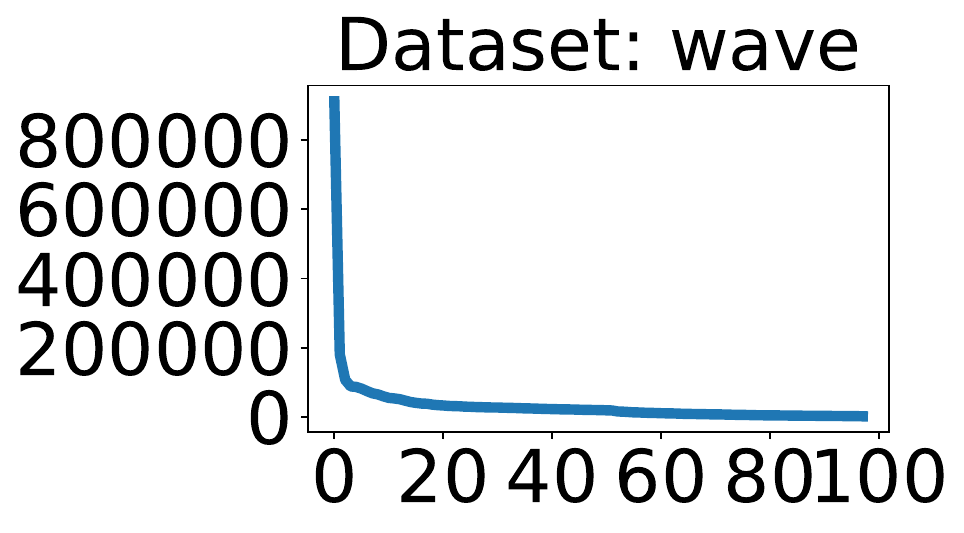}
    \includegraphics[width=0.35\linewidth]{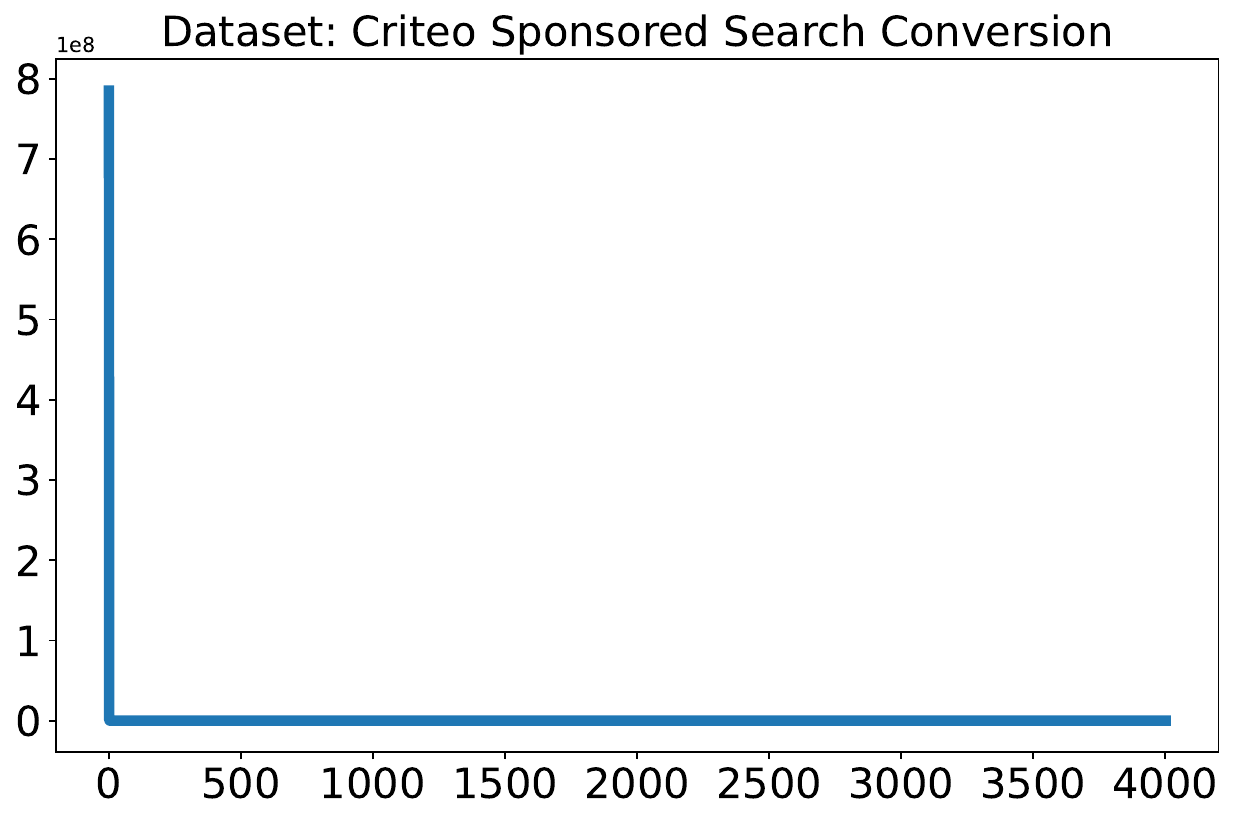}
    
    \caption{Singular value spectra of datasets used in the experiments.}
    \label{fig:spectra_datasets}
\end{figure}

\subsection{Full Experiment Results}
\label{subsec:appendix_full_exp_res}

We note that the baseline \textsf{Weighted-LLP}~\cite{brahmbhatt2023ldp_aggregation} requires the number of samples aggregated per bag to scale at least on the order of
$1/\eps^4$ and inversely with the label range. As a result, the method becomes impractical for datasets with limited sample sizes, in high-privacy regimes with small $\epsilon$, or when the label range is large. For instance, \textsf{Weighted-LLP} is not applicable to the \texttt{Criteo Sponsored Search Conversion} dataset due to its large label range of 400.

In Table~\ref{tab:private_linear_regression_real_data_res_table}, \ref{tab:private_mlp_real_data_res_table}, \ref{tab:criteo_search_linear_model_5_runs} and \ref{tab:criteo_search_mlp_5_runs} we report the complete numerical results of our experiments and include \textsf{Weighted-LLP} only in settings where it applies.

\begin{table}[h]
    \centering
    \begin{adjustbox}{width=0.6\linewidth}
    
    \begin{tabular}{|c|c|c|c|c|c|}
    \hline
         Dataset & $\eps$  & \dpsgd\ & \rronbins\ & \wtdllp\ & \precond\ \\
    \hline
        \multirow{6}{*}{\shortstack{\texttt{Boston\_housing} \\ ($\times 10^{-2}$)}} 
        & 0.25 & 1.3503 (0.0169) & 1.3846 (0.0436) & -- & 1.3255 (0.0015) \\
        & 0.5  & 1.3450 (0.0102) & 1.3566 (0.0368) & 22.9958  & 1.3196 (0.0002) \\
        & 1    & 1.3412 (0.0091) & 1.3492 (0.0163) & 9.2981 & 1.3196 (0.0004) \\
        & 2    & 1.3274 (0.0042) & 1.3573 (0.0278) & 4.1275 & 1.3193 (0.0006) \\
        & 4    & 1.3272 (0.0040) & 1.3556 (0.0348) & 3.7440 & 1.3193 (0.0001) \\
        & $\infty$ & 1.3267 (0.0087) & -- & -- & 1.3190 (0.0001) \\
    \hline
        \multirow{6}{*}{\texttt{wine}} 
        & 0.25 & 0.7671 (0.0336) & 0.7513 (0.0386) & -- & 0.7437 (0.0484) \\
        & 0.5  & 0.7411 (0.0169) & 0.7514 (0.0205) & -- & 0.7225 (0.0240) \\
        & 1    & 0.7192 (0.0052) & 0.7448 (0.0160) & -- & 0.6663 (0.0102) \\
        & 2    & 0.6946 (0.0133) & 0.6837 (0.0022) & 4.6771 & 0.6457 (0.0067) \\
        & 4    & 0.6833 (0.0149) & 0.6324 (0.0003) & 1.7021 & 0.6321 (0.0003) \\
        & $\infty$ & 0.6802 (0.0053) & -- & -- & 0.6319 (0.0001) \\
    \hline
        \multirow{6}{*}{\shortstack{\texttt{energy} \\ ($\times 10^{-3}$)}} 
        & 0.25 & 8.2643 (0.2053) & 8.5292 (0.0064) & 3027.8 & 8.0378 (0.1372) \\
        & 0.5  & 8.2027 (0.1870) & 8.4731 (0.0023) & 55.5 & 7.9668 (0.1036) \\
        & 1    & 8.1793 (0.1011) & 8.0962 (0.0018) & 17.9 & 7.6883 (0.0832) \\
        & 2    & 8.0678 (0.0872) & 7.7093 (0.0006) & 11.6 & 7.5157 (0.0425) \\
        & 4    & 8.0339 (0.0206) & 7.3350 (0.0005) & 11.5 & 7.3320 (0.0581) \\
        & $\infty$ & 7.9260 (0.0080) & -- & -- & 7.2562 (0.0071) \\
    \hline
        \multirow{6}{*}{\texttt{CA\_housing}} 
        & 0.25 & 1.2643 (0.0281) & 1.2951 (3.17e-5) & -- & 0.9184 (0.0497) \\
        & 0.5  & 1.2561 (0.0357) & 1.2605 (1.62e-5) & -- & 0.7611 (0.0104) \\
        & 1    & 1.2316 (0.0232) & 1.1347 (3.56e-6) & 3.1290 & 0.6927 (0.0127) \\
        & 2    & 1.1881 (0.0284) & 0.8966 (1.82e-6) & 0.6694 & 0.6570 (0.0111) \\
        & 4    & 1.1321 (0.0215) & 0.6500 (6.48e-6) & 0.6635 & 0.6552 (0.0097) \\
        & $\infty$ & 1.0641 (0.0011) & -- & -- & 0.6480 (0.0004) \\
    \hline
    
    \end{tabular}
    \end{adjustbox}
    \caption{Validation MSEs (mean and std.\ in parenthesis over 5 runs) for private linear regression on real-world datasets under varying privacy budgets $\epsilon$.}
    \label{tab:private_linear_regression_real_data_res_table}
\end{table}

\begin{table}[h]
    \centering
    \begin{adjustbox}{width=0.6\linewidth}
    \begin{tabular}{|c|c|c|c|c|}
    \hline
         Dataset & $\eps$  & \dpsgd\ & \rronbins\ & \switch\ \\
    \hline
        \multirow{5}{*}{\texttt{supercond}} 
        & $0.25$ & $0.0181\,(0.0014)$ & $0.0330\,(0.0001)$ & $0.0172\,(0.0009)$ \\
        & $0.5$  & $0.0147\,(0.0011)$ & $0.0314\,(0.0001)$ & $0.0147\,(0.0008)$ \\
        & $1$    & $0.0122\,(0.0007)$ & $0.0249\,(0.0002)$ & $0.0114\,(0.0005)$ \\
        & $2$    & $0.0101\,(0.0006)$ & $0.0137\,(0.0003)$ & $0.0095\,(0.0005)$ \\
        & $4$    & $0.0093\,(0.0004)$ & $0.0065\,(0.0003)$ & $0.0086\,(0.0007)$ \\
    \cline{3-5}
          & $\infty$ & 0.0055 (0.0002) & -- & 0.0036 (0.0000) \\
    \hline
        \multirow{6}{*}{\texttt{TomsHardware}} 
          & $0.25$ & 0.0945 (0.0108) & 1.9202 (0.0308) & 0.0870 (0.0052) \\
          & $0.5$  & 0.0838 (0.0071) & 1.9251 (0.0366) & 0.0818 (0.0025) \\
          & $1$    & 0.0799 (0.0099) & 1.9071 (0.0666) & 0.0793 (0.0074) \\
          & $2$    & 0.0781 (0.0041) & 1.2699 (0.0026) & 0.0748 (0.0089) \\
          & $4$    & 0.0764 (0.0072) & 0.3229 (0.0004) & 0.0719 (0.0034) \\
    \cline{3-5}
          & $\infty$ & 0.0212 (0.0002) & -- & 0.0182 (0.0002) \\
    \hline
        \multirow{5}{*}{\texttt{wave}} 
        & $0.25$ & $0.0597\,(0.0089)$ & $0.0149\,(0.0002)$ & $0.0127\,(0.0012)$ \\
        & $0.5$  & $0.0132\,(0.0015)$ & $0.0142\,(0.0002)$ & $0.0088\,(0.0016)$ \\
        & $1$    & $0.0099\,(0.0009)$ & $0.0114\,(0.0001)$ & $0.0043\,(2.30e\!-\!04)$ \\
        & $2$    & $0.0080\,(0.0012)$ & $0.0066\,(0.0002)$ & $0.0030\,(4.85e\!-\!05)$ \\
        & $4$    & $0.0056\,(0.0009)$ & $0.0025\,(0.0002)$ & $0.0028\,(8.23e\!-\!05)$ \\
    \cline{3-5}
          & $\infty$ & 0.0019 (0.0003) & -- & 0.0012 (0.0002) \\
    \hline
    \end{tabular}
    \end{adjustbox}
    \caption{Validation MSEs (mean and std.\ in parenthesis over 5 runs) for private models with linear input transformations and a two-layer MLP prediction layer (16, 8 hidden units) under varying privacy budgets $\epsilon$. }
    \label{tab:private_mlp_real_data_res_table}
\end{table}

\begin{table}[h]
    \centering
    \begin{adjustbox}{width=0.5\textwidth}
    \begin{tabular}{|c|c|c|c|}
    \hline
        $\eps$ & \dpsgd & \precond & \rronbins (reported) \\
    \hline
        0.25 & 6842.63 (15.63) & 9677.99 (97.21) & -- \\
        0.5  & 6780.06 (107.82) & 6998.62 (103.04) & 10901.33 (36.54) \\
        1    & 6673.56 (10.84) & 5630.08 (46.49) & 9744.08 (37.59) \\
        2    & 6589.69 (4.86) & 5335.86 (37.92) & 7294.93 (34.03) \\
        4    & 6565.11 (34.31) & 5297.63 (12.28) & 4769.61 (25.01) \\
        $\infty$ & 5287.43 (3.87) & 4726.78 ($1.98 \times 10^{-5}$) & -- \\
    \hline
    \end{tabular}
\end{adjustbox}
    \caption{Validation MSEs (mean and std.\ in parenthesis over 5 runs) on \texttt{Criteo Sponsored Search Conversion} using linear models under varying privacy budgets $\epsilon$. Results for \rronbins\ are taken from~\cite{ghazi2023regression_ldp} when available.}
\label{tab:criteo_search_linear_model_5_runs}
\end{table}

\begin{table}[h]
    \centering
    \begin{adjustbox}{width=0.5\textwidth}
    \begin{tabular}{|c|c|c|c|}
    \hline
        $\eps$ & \dpsgd & \switch & \rronbins (reported) \\
    \hline
        0.25 & 8533.74 (2001.51) & 6975.90 (11.63) & -- \\
        0.5  & 8073.85 (883.22) & 6959.93 (15.74) & 10901.33 (36.54) \\
        1    & 7472.65 (1022.24) & 6935.00 (47.36) & 9744.08 (37.59) \\
        2    & 6933.28 (324.31) & 6529.16 (45.53) & 7294.93 (34.03) \\
        4    & 6444.03 (405.99) & 6365.98 (75.15) & 4769.61 (25.01) \\
        $\infty$ & 4296.44 (87.26) & 4128.41 (23.31) & -- \\
    \hline
    \end{tabular}
    \end{adjustbox}
    \caption{Validation MSEs (mean and std.\ in parenthesis over 5 runs) on \texttt{Criteo Sponsored Search Conversion} using models with linear input transformations and a two-layer MLP prediction head (128, 64 hidden units) under varying privacy budgets $\epsilon$. Results for \rronbins\ are taken from~\cite{ghazi2023regression_ldp} when available.}
    \label{tab:criteo_search_mlp_5_runs}
\end{table}


\end{document}